\documentclass{article}

\usepackage[table]{xcolor} % For alternating row colors
\definecolor{cvprblue}{rgb}{0.21,0.49,0.74}
\usepackage[pagebackref,breaklinks,colorlinks,allcolors=cvprblue]{hyperref}
\PassOptionsToPackage{numbers, compress}{natbib}

\usepackage{array}

\definecolor{OpenSrcBlue}{rgb}{0.921,0.961,1.0}
\newcolumntype{O}{>{\columncolor{OpenSrcBlue}}c}
\newcommand{\avgrow}{\rowcolor{OpenSrcBlue}}

 \usepackage[preprint]{neurips_2026}

\usepackage[utf8]{inputenc} % allow utf-8 input
\usepackage[T1]{fontenc}    % use 8-bit T1 fonts
\usepackage{hyperref}       % hyperlinks
\usepackage{url}            % simple URL typesetting
\usepackage{booktabs}       % professional-quality tables
\usepackage{amsfonts}       % blackboard math symbols
\usepackage{nicefrac}       % compact symbols for 1/2, etc.
\usepackage{microtype}      % microtypography
\usepackage{xcolor}         % colors

\usepackage{graphicx}
\usepackage{booktabs}

\usepackage{amsmath}
\usepackage{multirow}

\newcommand{\mstd}[2]{#1$_{#2}$}

\usepackage{algorithm}
\usepackage{algpseudocode}

\usepackage{enumitem}

\usepackage{booktabs} % For professional looking tables
\usepackage{multirow}
\usepackage{amsthm}

\newtheorem{proposition}{Proposition}
\newtheorem{remark}{Remark}
\newtheorem{theorem}{Theorem}

\usepackage{caption}

\definecolor{lightblue}{RGB}{235,245,255} % Soft blue tint
\definecolor{headerpurple}{RGB}{230,238,250}

\usepackage{makecell} % Add this to your preamble if not already there

\usepackage{booktabs}
\usepackage{multirow}
\usepackage{makecell}
\usepackage{threeparttable}
\usepackage{siunitx}
\usepackage{pifont}    % for checkmarks

\newcommand{\ourmodel}{CorrFlow}

\title{
Correlation-Guided Flow Matching with Annealed Masking for Spatial Transcriptomics Generation
}

\author{%
  Yupei Zhang \\
  Department of Clinical Neurosciences \\
  University of Cambridge \\
  UK \\
  \texttt{yz931@cam.ac.uk} \\
  \And
  Hao Chen \\ 
  Department of Clinical Neurosciences \\ 
  University of Cambridge \\ 
  UK \\ 
  \texttt{hc666@cam.ac.uk} \\
  \And
  Li Pan \\ 
  Usher Institute \\ 
  The University of Edinburgh \\ 
  EH16 RUX \\ 
  \texttt{l.pan-13@sms.ed.ac.uk} \\
  \And
  Chao Li\textsuperscript{\ensuremath{\dagger}} \\ 
  Department of Clinical Neurosciences \\ 
  University of Cambridge \\ 
  UK \\ 
  \texttt{cl647@cam.ac.uk} \\
  \And
  Xiaohan Xing\textsuperscript{\ensuremath{\dagger}} \\
  Department of Diagnostic Radiology \\ 
  National University of Singapore \\ 
  Singapore \\ 
  \texttt{xhxing@nus.edu.sg} \\
}

\begin{document}

\maketitle

\begingroup
\renewcommand{\thefootnote}{\fnsymbol{footnote}}
\footnotetext[2]{Corresponding authors.}
\endgroup

\begin{abstract}
Spatial transcriptomics (ST) provides spatially resolved gene expression profiling but remains expensive, motivating the prediction of ST from histology images. Generative models have emerged as a mainstream paradigm for ST prediction due to their ability to model the conditional distribution of gene expression and capture its inherent stochasticity.
However, these methods typically treat genes as independent prediction targets and overlook the intrinsic gene-gene interactions in biological systems, which limits their ability to preserve biologically meaningful co-expression patterns.
We argue that gene-gene interactions, which reflect shared pathways and regulatory mechanisms, are essential for generating numerically accurate and biologically coherent ST profiles.
In this paper, we propose \textbf{\ourmodel}, a correlation-guided flow matching framework for histology-to-ST prediction that explicitly models gene-gene dependencies through two complementary mechanisms.
First, we introduce an \textit{annealed masked flow matching} strategy, where subsets of genes are progressively masked following a timestep-dependent annealing schedule, encouraging the model to infer masked genes conditioned on the remaining genes and promoting joint conditional modeling beyond per-gene marginal estimation.
Second, we devise a \textit{gene graph-regularized optimization} scheme that integrates prior knowledge from the STRING database and data-driven co-expression estimated by WGCNA to construct a gene affinity graph, which enforces both local consistency and global smoothness in the predicted expression.
% Extensive experiments across 12 datasets demonstrate that \textbf{\ourmodel} improves both per-gene prediction accuracy and gene-gene correlation fidelity, leading to more biologically coherent ST predictions.
Extensive experiments across 12 datasets show that \textbf{\ourmodel} achieves the best average PCC and HPCC among evaluated methods, leading to more biologically coherent ST predictions.
% Spatial transcriptomics (ST) provides spatially resolved gene expression profiling but remains expensive and low-throughput, motivating the prediction of ST from histology images. Generative models are well-suited for this task as they capture the inherent uncertainty and stochasticity of gene expression. However, existing approaches largely treat genes as independent targets, failing to account for the structured biological dependencies among genes. We argue that gene interactions are not a nuisance to model around but essential for biologically faithful generation. We introduce \ourmodel~, which reformulates flow matching from marginal estimation to conditional joint modeling over co-regulated gene sets. The core mechanism is masked flow matching: at each step, a subgroup of genes is masked and recovered conditional on the remaining observed genes, breaking the per-gene factorization that drives prior methods toward marginal solutions. To further enforce the learned co-expression structure, we construct a gene affinity graph derived from regulatory priors that supervises the output dependency structure, ensuring the model is anchored to biologically meaningful dependencies rather than arbitrary co-occurrence. Across 12 datasets, \ourmodel~improves joint distribution fidelity by \red{[X\%]} and recovers co-expression structure that prior methods systematically destroy-establishing joint modeling, rather than pointwise accuracy, as the correct objective for transcriptomics generation.
\end{abstract}

% Introduction (outline)
% Problem–Gap–Solution–Contrib

\section{Introduction}
\label{sec:intro}

%%% original version
% Spatial transcriptomics (ST) measures gene expression at spatially resolved locations across tissue, offering a molecular complement to histology. Yet ST assays are expensive and low-throughput, which has motivated a line of work that predicts spatial gene expression directly from H\&E-stained whole-slide images~\cite{wang2025benchmarking}. Early efforts adopt deterministic regression~\cite{he2020integrating, zeng2022spatial, pang2021leveraging, chung2024accurate}, learning a direct mapping from image patch to expression vector and producing a single point estimate per location. Retrieval-based methods~\cite{xie2023spatially, yang2023exemplar} instead predict expression by matching query patches against a reference set, thereby enhancing trustworthiness.  More recent generative approaches~\cite{ICLR2025_31cc93d1, pmlr-v267-huang25t} model the full conditional distribution of expression given histology, enabling more faithful predictions. 

%% xiaohan revised
Spatial transcriptomics (ST) measures gene expression at spatially resolved locations across tissue, providing a molecular complement to histology and enabling applications such as spatial domain identification and tumor microenvironment analysis~\cite{marx2021method, staahl2016visualization, hu2021spagcn, biancalani2021deep, chen2023cell, zhang2024inferring, xiao2021tumor, de2023evolving}.
However, current ST assays remain expensive and low-throughput, motivating efforts to predict ST from H\&E-stained histology images~\cite{wang2025benchmarking}.
Early approaches adopt deterministic regression~\cite{he2020integrating, zeng2022spatial, pang2021leveraging, chung2024accurate}, learning a direct mapping from image patches to expression vectors, but failing to capture the inherent uncertainty and stochasticity of gene expression.
Retrieval-based methods~\cite{xie2023spatially, yang2023exemplar} predict expression by matching query patches against a reference set but depend heavily on its diversity and coverage.
Recently, generative approaches~\cite{ICLR2025_31cc93d1, pmlr-v267-huang25t} have emerged as a promising paradigm for ST prediction by modeling the conditional distribution of gene expression and capturing its inherent stochasticity.

\textbf{Motivation.} Despite their advantages, current generative approaches~\cite{ICLR2025_31cc93d1, pmlr-v267-huang25t} leave a fundamental issue insufficiently addressed: genes are typically treated as independent targets, despite their rich and structured biological dependencies. In reality, gene expression is governed by regulatory networks, co-expression modules, and coordinated pathway activity, forming highly interdependent systems~\cite{muzio2021biological, wang2019pathway, yan2025pathway, jaume2024modeling, sun2024sprite}. However, most existing models lack explicit mechanisms to enforce such gene-gene relationships.
We show in Section~\ref{sec:theory} that this limitation is intrinsic to the training objective. Although generative models define a joint distribution over genes, standard diffusion (Fig.~\ref{fig:comparison} Diffusion) and flow matching (Fig.~\ref{fig:comparison} Flow Matching) objectives decompose across dimensions, providing no explicit incentive to capture cross-gene dependencies.
This can lead to uneven and uncoordinated performance across genes, where some genes are accurately predicted (e.g., $x_1^1$ and $x_1^2$) while other highly correlated genes (e.g., $x_1^3$ ) remain poorly estimated. 
As a result, the model fails to effectively capture biologically meaningful correlation patterns, limiting biological fidelity and downstream utility.
This motivates explicitly modeling gene-gene interactions within generative frameworks to promote more coherent and biologically consistent predictions.

% \textbf{Our approach.} We address this challenge with two complementary mechanisms that target the same failure mode. First, we observe that standard flow matching admits a trivial solution at the per-gene level: the model can simply propagate each gene's input value through the velocity field without learning any cross-gene structure. To force the model to capture co-expression, we propose \emph{annealed masked flow matching}, in which a subset of genes is masked at the input and the model must recover them from the remaining genes. This breaks the per-gene factorization: a masked gene cannot be predicted from its own input dimension and must instead be inferred from its co-expressed neighbors, making joint modeling \emph{necessary for optimality}. Second, we construct a gene affinity graph from reference expression data and use it to guide the generation, ensuring that the gene-gene correlations the model learns remain biologically coherent under finite capacity.

\begin{figure}[t]
    \centering
    \includegraphics[scale=.36]{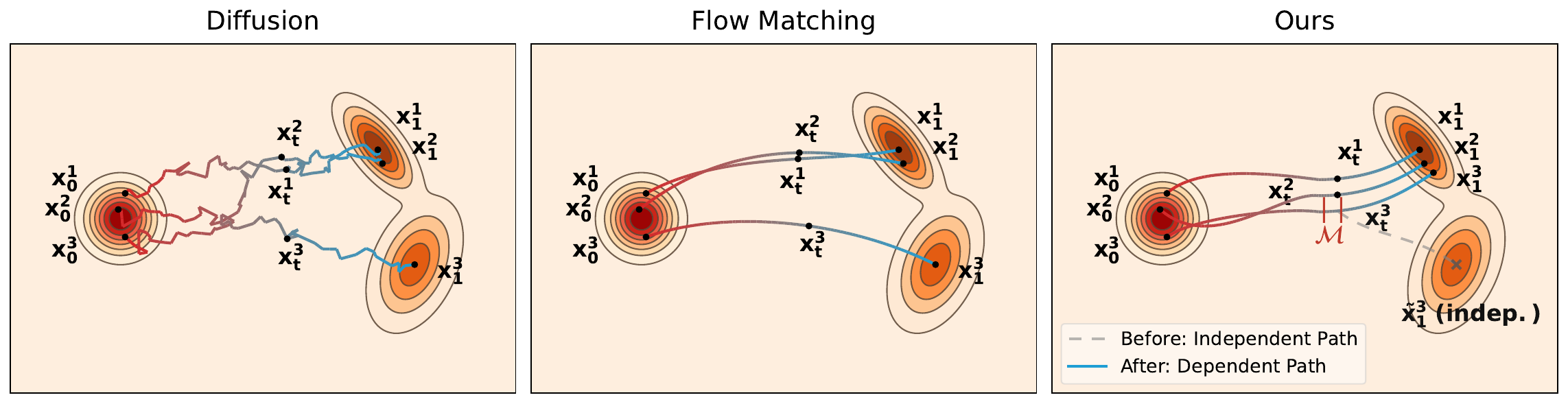}
    \caption{\textbf{Correlation-guided Flow Matching enables joint trajectory inference for correlated genes.} Given correlated genes with source states $\{x_0^i\}_{i=1}^3$, existing methods such as Diffusion and Flow Matching treat each trajectory independently, failing to capture inter-gene dependencies and producing inconsistent predictions $\tilde{x}_1^3$. Our method applies an annealed masking $\mathcal{M}$ and an explicit graph regularization, enforcing dependency across trajectories of genes $1$ and $2$. This dependency-aware guidance steers $x_1^3$ toward a coherent joint prediction consistent with the target distribution. 
    % \textcolor{red}{[Xiaohan: In the current description, the prediction of $x_1^3$ appears to benefit solely from masking. I would suggest attributing this improvement to both masking and the explicit graph regularization.]}
    % \hao{The text ($x_0, x_1$, etc.) is a bit hard to see in the figure. Consider shifting it a bit, using another color, or making it bold.}
    }
    \label{fig:comparison} 
\end{figure}

% xiaohan revised.
% \textbf{Our approach.} In this paper, we present \textbf{\ourmodel}, a novel conditional flow matching framework that incorporates gene-gene interactions for histology-to-ST prediction.
In this paper, we propose \textbf{\ourmodel}, a correlation-guided flow matching framework for histology-to-ST prediction that explicitly models gene-gene dependencies through two complementary mechanisms.
% \textcolor{red}{[Xiaohan: I would suggest renaming the model to CorrSTFlow, as FlowST is too similar to STFlow and does not reflect the core idea. Please update the name consistently throughout the paper.]}. 
% The model consists of two complementary modules to enforce cross-gene joint modeling.
% First, to encourage the model to capture co-expression structure, we propose \emph{annealed masked flow matching}, where a subset of genes is masked at the input and must be inferred from the remaining genes.  
% By progressively masking each gene's own expression input during flow matching on the generated ST profiles, the model is encouraged to exploit cross-gene relationships rather than rely on independent per-gene prediction, as illustrated in Fig.~\ref{fig:comparison} (Ours).
First, we propose \emph{annealed masked flow matching}, where subsets of genes are progressively masked according to the flow timestep $t$ and must be inferred from the remaining genes. As the masking ratio increases with $t$, the model is increasingly encouraged to exploit cross-gene relationships rather than rely on independent per-gene prediction, as illustrated in Fig.~\ref{fig:comparison} (Ours).
% \textcolor{red}{[Xiaohan: You may need to mention the annealed masking strategy here, as it is the novelty of this part. Otherwise, the masking seems to be too simplified.]}
Second, we introduce a \emph{gene graph-regularized optimization} scheme to explicitly incorporate gene-gene correlations. We construct a gene affinity graph from two complementary sources:  
the STRING database~\cite{szklarczyk2025string}, which provides functional associations between genes, and WGCNA~\cite{rezaie2023pywgcna}, which estimates data-driven gene co-expression.
The fused graph is then used to regularize the predicted expression with two terms: a Huber-based term for local consistency and a Laplacian term for global smoothness. Together, these components constrain the output space and improve the biological coherence of the generated gene expression. The contributions of this work are summarized as:

\begin{itemize}[leftmargin=*]
    % \item We identify a fundamental limitation of current generative models for ST prediction: genes are predicted independently without explicitly modeling gene–gene interactions, limiting the ability to capture biologically meaningful co-expression structure.
    \item We identify a key limitation of existing generative models for ST prediction, their lack of explicit modeling of intrinsic gene-gene correlations, and propose \ourmodel, a conditional flow matching framework that explicitly models gene dependencies for histology-to-ST prediction.
    \item We introduce an \textit{annealed masking strategy} conditioned on the timestep $t$, which progressively shifts the learning objective from per-gene marginal estimation to joint conditional modeling.
    % \item We integrate external functional prior and data-driven co-expression to construct a gene affinity graph, and incorporate it into the optimization via Huber-based local consistency and Laplacian-based global smoothness regularization, leading to more biologically coherent ST predictions.
    \item We introduce a \textit{gene graph-regularized optimization} scheme that integrates prior biological knowledge and data-driven co-expression graphs to explicitly enforce structured gene dependencies during ST generation.
    % \item Across 12 ST datasets, \ourmodel~improves both per-gene accuracy and gene-gene correlation fidelity, and recovers co-expression structure that prior methods systematically destroy.
    % \item Across 12 ST datasets, \ourmodel~achieves the best gene-wise PCC among SOTA methods and recovers co-expression structure, demonstrating the effectiveness of dependency-aware flow matching for histology-conditioned ST generation.
    % \item Across 12 ST datasets, \ourmodel~achieves state-of-the-art gene-wise PCC and gene correlation preservation, highlighting the importance of gene dependency modeling in ST prediction.
    \item Across 12 ST datasets, \ourmodel~achieves the best average PCC and HPCC among evaluated methods, highlighting the importance of gene dependency modeling in ST prediction.

    % \item We propose \textbf{\ourmodel}, a conditional flow matching framework that incorporates gene-gene interactions via two complementary modules: \emph{annealed masked flow matching} for cross-gene contextual learning and \emph{gene graph-regularized optimization} for enforcing biologically coherent dependencies.

\end{itemize}

\section{Related Work}
\label{sec:relatedwork}

% Our work relates to two lines of research: spatial gene generation and gene-gene structure modeling.

\paragraph{Histology-to-ST Prediction.}

Early work on predicting ST from histology largely formulated the task as supervised regression from image patches to spot-level expression values, optionally combined with spatial coordinates. ST-Net~\cite{he2020integrating} first demonstrated that spatial gene expression can be predicted from histology. Subsequent regression-based methods enhanced this pipeline with stronger spatial and contextual modeling to capture local and long-range dependencies~\cite{zeng2022spatial, pang2021leveraging, chung2024accurate, wang2025m2ost}. However, regression-based approaches typically produce deterministic point estimates and fail to capture the inherent uncertainty and one-to-many nature of gene expression. Retrieval-based methods~\cite{xie2023spatially, yang2023exemplar,wang2025fmh2st, hu2026histoprism} instead align histology and expression in a shared embedding space, and predict ST by retrieving similar patches from a reference set. However, their performance depends heavily on the diversity and coverage of the reference set. 

More recently, there has been a shift toward generative modeling, motivated by its ability to capture uncertainty and the inherent one-to-many nature of gene expression. For example, Stem~\cite{ICLR2025_31cc93d1} formulates expression prediction as conditional diffusion, and STFlow~\cite{pmlr-v267-huang25t} as conditional flow matching over whole-slide spot collections, jointly modeling expression across spots with improved sampling efficiency. GenAR~\cite{ouyang2025genar} generates expression autoregressively over gene groups, and STPath~\cite{huang2025stpath} adopts mask-based generative pretraining.
% However, across regression, retrieval, and generative formulations, gene-gene correlations are not explicitly modeled, as existing approaches primarily focus on per-gene prediction or marginal distributions. This limitation is consistent with the formulation of current training objectives and motivates our work to explicitly incorporate gene-gene interactions into generative modeling.
However, existing methods primarily optimize per-gene predictions or marginal distributions without explicitly modeling gene-gene dependencies, limiting their ability to preserve biologically meaningful co-expression structure. This limitation motivates our dependency-aware generative framework for ST prediction.

\paragraph{Gene-Gene Interaction Modeling.} Gene expression is organized by co-expression modules, regulatory programs, and interaction networks. Gene-gene interactions can be modeled using either prior-based or data-driven approaches.
For instance, WGCNA~\cite{rezaie2023pywgcna} constructs gene co-expression networks in a data-driven manner by identifying modules of highly correlated genes. STRING~\cite{szklarczyk2025string} provides curated and predicted protein-protein interaction networks that capture functional associations between genes.
Incorporating such structures has been shown to improve molecular prediction and interpretation~\cite{shi2025multi}. 
This idea is especially appealing for histology-to-ST generation, where biologically plausible outputs should preserve meaningful cross-gene structure rather than merely optimize per-gene accuracy~\cite{hu2026histoprism}. However, such priors are typically applied in regression models or feature encoders, and are rarely integrated into the dynamics of conditional generative learning. 
Our work takes a step toward incorporating structured gene dependencies into conditional generative frameworks, enabling improved biological fidelity and coherence.

\section{Methodology}
\label{sec:method}

% \textcolor{red}{[Xiaohan: Please add a paragraph here to introduce the framework. As shown in Fig. 2, xxx, provide a brief overview for the whole pipeline]}

% Motivated by the per-gene decomposition limitation identified in Section~\ref{sec:theory}, we introduce the framework illustrated in Fig.~\ref{fig:framework}, which generates spatial gene expression through two complementary components: \textit{Annealed Masked Flow Matching} (Section~\ref{sec:mfm}) masks gene dimensions in the input to enforce joint cross-gene modeling, while a \textit{Gene Graph-Regularized Optimization} prior (Section~\ref{sec:genegraph}) regularizes predictions toward biologically coherent structure.

% \subsection{Overview.}
Motivated by the per-gene decomposition limitation identified in Section~\ref{sec:theory}, we introduce \ourmodel, a dependency-aware conditional flow matching framework for histology-to-ST prediction. As shown in Fig.~\ref{fig:framework}, \ourmodel~combines \textit{Annealed Masked Flow Matching} (Section~\ref{sec:mfm}) and \textit{Gene Graph-Regularized Optimization} (Section~\ref{sec:genegraph}) to encourage cross-gene modeling. 
Given histology features $z$, spatial coordinates $q$, timestep $t$, and the interpolant $x_t=(1-t)x_0+t x_1$ between a zero-inflated negative binomial (ZINB) prior sample $x_0$ and ground-truth expression $x_1$, we obtain the masked input $\hat{x}_t$ using the timestep-dependent masking ratio $p(t)=p_{\max}\cdot t$. A spatial-transformer denoiser~\cite{pmlr-v267-huang25t} then predicts the target expression from $(z,q,t,\hat{x}_t)$ and is optimized with both reconstruction and gene graph-based regularization.

\subsection{Theoretical Background}
\label{sec:theory}

% We begin by formalizing a structural limitation of standard flow matching~\cite{lipman2023flow, albergo2022building, liu2022flow} (and, by extension, score-based diffusion) when applied to multi-gene expression generation. The results in this section are stated for flow matching; analogous statements hold for score-based diffusion via the connection between velocity fields and score functions \cite{lipman2023flow}.

We begin by formalizing a structural limitation of standard flow matching~\cite{lipman2023flow, albergo2022building, liu2022flow} when applied to multi-gene expression generation. 
% Analogous conclusions hold for score-based diffusion through the connection between velocity fields and score functions~\cite{lipman2023flow}.

\textbf{Setup and Notation.}  
Let $x_1 \in \mathbb{R}^G$ denote a gene expression vector with $G$ genes at a single spatial spot, and let $c$ denote the conditioning information (histology features, spatial coordinates).
In conditional flow matching (CFM), we draw a prior sample $x_0 \sim p_0$ and form the linear interpolant
\begin{equation}
  x_t = (1-t)\,x_0 + t\,x_1, \qquad t \in [0,1].
  \label{eq:interpolant}
\end{equation}
A velocity network $v_\theta(x_t, t, c) \in \mathbb{R}^G$ is trained to predict the conditional velocity $u_t = x_1 - x_0$ via
\begin{equation}
  \mathcal{L}_{\mathrm{CFM}}(\theta)
  = \mathbb{E}_{t,\,x_0,\,x_1,\,c}
    \bigl\|
      v_\theta(x_t, t, c) - (x_1 - x_0)
    \bigr\|^2.
  \label{eq:cfm_loss}
\end{equation}
In practice, many methods predict endpoint $x_1$ directly
rather than the velocity; the two formulations are equivalent up to the affine relation $x_1 = x_0 + u_t$. We use the endpoint-prediction form below.

\begin{figure}[t]
    \centering
    \includegraphics[width=\textwidth]{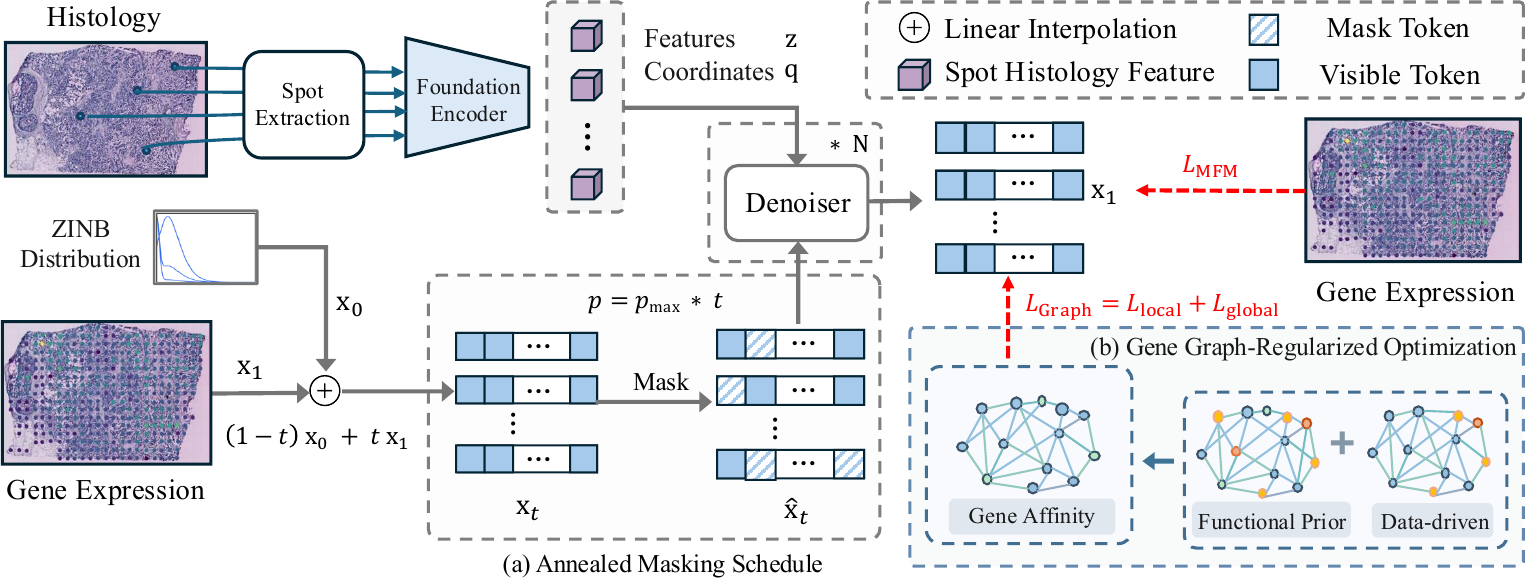}
    \caption{\textbf{Overview of the proposed \ourmodel.} We model histology-to-ST inference as conditional flow matching conditioned on histology features and spatial coordinates. To promote gene-dependent generation, we introduce (a) an \textit{annealed masking schedule} conditioned on timestep t and (b) a \textit{gene graph-regularized optimization} strategy, which constructs gene affinity by fusing an external functional prior with data-driven co-expression, and imposes a Huber-based \emph{local smoothness} and a Laplacian \emph{global smoothness} to improve the biological coherence of generated expression.}
    \label{fig:framework}
\end{figure}

\begin{proposition}[Per-gene decomposition of the standard objective]
\label{prop:decomp}
  Let $v_\theta(x_t, t, c) = (v_\theta^1, \dots, v_\theta^G)$ denote
  the $G$-dimensional output of the velocity network.
  The  objective~\eqref{eq:cfm_loss} decomposes
  as
  \begin{equation}
    \mathcal{L}_{\mathrm{CFM}}(\theta)
    = \sum_{g=1}^{G}
      \mathbb{E}_{t,\,x_0,\,x_1,\,c}
      \bigl(
        v_\theta^g(x_t, t, c) - u_t^g
      \bigr)^2
    = \sum_{g=1}^{G} \mathcal{L}_{\mathrm{CFM}}^{(g)}(\theta),
    \label{eq:decomp}
  \end{equation}
  where $u_t^g = x_1^g - x_0^g$ is the $g$-th component of the target velocity.
  Consequently, the optimal velocity field satisfies, for each gene
  $g \in \{1,\dots,G\}$ independently,
  \begin{equation}
    v^{*,g}(x_t, t, c)
    = \mathbb{E}\bigl[u_t^g \mid x_t, t, c\bigr]
    = \mathbb{E}\bigl[x_1^g \mid x_t, t, c\bigr] - x_0^g.
    \label{eq:optimal_gene}
  \end{equation}
\end{proposition}

 \begin{proof}
  The squared $\ell_2$ norm decomposes over coordinates:
  $\|v_\theta - u_t\|^2 = \sum_g (v_\theta^g - u_t^g)^2$.
  Because expectation is linear, $\mathcal{L}_{\mathrm{CFM}}$ is the sum of $G$ independent per-gene losses.
  Minimizing each $\mathcal{L}_{\mathrm{CFM}}^{(g)}$ pointwise yields $v^{*,g} = \mathbb{E}[u_t^g \mid x_t, t, c]$, which depends on $x_1^g$ only through its marginal conditional distribution given $(x_t, t, c)$.
\end{proof}

\begin{remark}[Extension to diffusion models]
  The same decomposition applies to score-based diffusion.
  The denoising score matching objective
  $\mathbb{E}\|s_\theta(x_t,t,c) - \nabla_{x_t}\log p_t(x_t|x_1)\|^2$ decomposes identically because the Gaussian transition kernel factorizes over coordinates and the squared norm decomposes over coordinates.
\end{remark}

Proposition~\ref{prop:decomp} does not claim that the network \emph{cannot} represent cross-gene dependencies; a shared backbone with sufficient capacity can, in principle, share intermediate representations across genes.
  The proposition states something fundamental: the \emph{training objective} never rewards such sharing.
  Any improvement in gene $g$'s prediction that arises from modeling its correlation with gene $j$ is equally achievable by a predictor that ignores gene $j$ entirely, because the loss for gene $g$ depends only on the marginal $p(x_1^g \mid x_t, t, c)$.
  Cross-gene coordination is therefore \emph{unidentifiable} under the standard objective that it may emerge incidentally but is never incentivized.

\subsection{Annealed Masked Flow Matching}
\label{sec:mfm}
% \textcolor{red}{[Xiaohan: The module name seems to be too broad and cannot reflect the novelty of annealed masking scheme. Maybe change to Annealed Masked Flow Matching?]}

The decomposition in Proposition~\ref{prop:decomp} arises because every gene dimension of the target $u_t$ has a corresponding gene dimension in the input $x_t$, allowing each gene to be predicted independently. 
We proposed \textit{annealed masked flow matching} (Annealed-MFM) to break this correspondence: at each step, a subset of gene dimensions is withheld from the input, requiring the model to infer the missing genes from the remaining ones.
We show that this simple modification provably shifts the optimal solution from per-gene marginals to conditional joint distributions.

\smallskip\textbf{Formulation.}
Let $\mathcal{M} \subseteq \{1,\dots,G\}$ denote a mask set sampled from a distribution $q(\mathcal{M})$ at each training step, and let $\mathcal{O} = \{1,\dots,G\} \setminus \mathcal{M}$.
The model receives a partially masked input $\tilde{x}_t$ in which the entries for $\mathcal{M}$ are replaced by learnable gene-wise mask tokens $m \in \mathbb{R}^G$, shared across spatial locations:
\begin{equation}
  \tilde{x}_{t,g}
  = \begin{cases}
      m_g, & g \in \mathcal{M}, \\
      x_{t,g}, & g \in \mathcal{O}.
    \end{cases}
  \label{eq:mask_op}
\end{equation}
% The masked flow matching objective over the masked genes is
% \begin{equation}
%   \mathcal{L}_{\mathrm{MFM}}(\theta; \mathcal{M})
%   = \mathbb{E}_{t,\,x_0,\,x_1,\,c}
%     \sum_{g \in \mathcal{M}}
%     \bigl(
%       v_\theta^g(\tilde{x}_t, t, c) - u_t^g
%     \bigr)^2.
%   \label{eq:masked_loss}
% \end{equation}

Crucially, masking operates exclusively on the \emph{input} $x_t \to \tilde{x}_t$; the \emph{target} $u_t$ and the \emph{supervision signal} remain unchanged across all $G$ genes.
The Annealed-MFM objective is therefore
\begin{equation}
  \mathcal{L}_{\mathrm{MFM}}(\theta; \mathcal{M})
  = \mathbb{E}_{t,\,x_0,\,x_1,\,c}\;
    \sum_{g=1}^{G}
    \bigl(
      v_\theta^g(\tilde{x}_t, t, c) - u_t^g
    \bigr)^2.
  \label{eq:masked_loss}
\end{equation}

% No masking is applied at inference time; the model receives the full
% $x_t$ and generates expression through standard ODE integration.

% \smallskip\textbf{Asymmetry between masked and observed genes.}
% Although the loss supervises all $G$ genes equally, the nature of the
% prediction task differs between the two groups.
% For observed genes $g \in \mathcal{O}$, the model retains access to
% $x_{t,g}$ and the prediction task resembles standard flow matching,
% albeit from a partially corrupted context. 
% For masked genes $g \in \mathcal{M}$, the input carries no
% information about $x_{t,g}$; the model must infer $u_t^g$ entirely
% from the observed genes, the conditioning $c$, and whatever
% inter-gene structure it has learned.  

% It is this asymmetry that shifts the optimal solution for the masked genes from per-gene marginals to the joint conditional, as we formalize next.

\begin{proposition}[Joint conditional modeling under masking]
\label{prop:joint} 
Although the loss supervises all $G$ genes equally, the nature of the prediction task differs between the two groups, producing the asymmetry that shifts the optimal solution for the masked genes from per-gene marginals to the joint conditional. 
  Under the masked objective~\eqref{eq:masked_loss}, the optimal predictor for \emph{every} gene $g \in \{1,\dots,G\}$ satisfies
  \begin{equation}
    v^{*,g}(\tilde{x}_t, t, c)
    = \mathbb{E}\bigl[
        u_t^g
        \;\big|\;
        x_t^{(\mathcal{O})},\, t,\, c
      \bigr].
    \label{eq:optimal_masked}
  \end{equation}
  For observed genes $g \in \mathcal{O}$, the conditioning set $x_t^{(\mathcal{O})}$ includes $x_{t,g}$ itself, so the predictor retains direct information about $x_{1,g}$ and behaves similarly to standard flow matching.
  For masked genes $g \in \mathcal{M}$, the conditioning set contains {no} information about $x_{t,g}$; the optimal prediction instead depends on the \textbf{joint conditional distribution} $p(x_1^{(\mathcal{M})} \mid x_1^{(\mathcal{O})}, c)$. If genes in $\mathcal{M}$ are correlated with those in $\mathcal{O}$, the optimal prediction for any masked gene $g$ depends on the co-variation structure between the two sets.
\end{proposition}
We provide the proof of Proposition~\ref{prop:joint} in the Appendix Section~\ref{sec:proof_proposition2}.

\subsubsection{Implementation Details}
\label{sec:mask_schedule}

\smallskip\textbf{Annealed masking schedule.}
Conditional flow matching suffers from a practical shortcut issue: as $t \to 1$, the interpolant $x_t \to x_1$, allowing the network to shortcut by copying near-target values rather than learning meaningful transport dynamics.
We suppress this by coupling the masking ratio to the timestep:
\begin{equation}
  p(t) = p_{\max} \cdot t,
  \label{eq:mask_rate}
\end{equation}
where $p_{\max} \in (0,1)$.
At small $t$, $x_t$ is far from $x_1$ and little masking is needed ($\tilde{x}_t \approx x_t$); at large $t$, heavy masking removes the shortcut signal ($\tilde{x}_t \approx m$), forcing the network to reconstruct $u_t$ from cross-gene context rather than copying.
In practice, we sample $t \sim \mathrm{Uniform}(0,1)$ and draw a Bernoulli mask independently per gene with probability $p(t)$. 
% We set $p_{\max} = 0.75$ as the default.

\smallskip\textbf{Learnable mask tokens.}
A natural masking strategy is to replace masked genes with a fixed zero token ($m_g = 0$). However, this introduces a systematic shift in input magnitude: during training, a fraction $p(t)$ of entries are zeroed, reducing the average input norm by a factor of approximately $1 - p(t)$. At inference, no masking is applied and all entries carry real values, resulting in a train-inference magnitude mismatch analogous to the dropout scaling problem~\citep{srivastava2014dropout}.

To mitigate this issue, we introduce a learnable mask token $m \in \mathbb{R}^G$, initialized to zero and optimized via backpropagation. Rather than acting as a fixed placeholder, the token adapts to preserve the input statistics expected by the network across different masking ratios. 
In effect, the network jointly learns \emph{what} to substitute for missing genes (the token value) and \emph{how} to use the remaining genes (the velocity predictor), enabling stable optimization without explicit rescaling during inference.
% adapting both to maintain consistent internal activation scales across masking ratios. Thus, we maintain clean inference with no masking required. 
We provide the convergence analysis of Annealed-MFM in the Appendix Section~\ref{sec:convergence}.

% \subsubsection{Inference Scheduler}

% ============================================================
%  BLOCK 2: Gene Affinity Graph
% ============================================================
 
\subsection{Gene Graph-Regularized Optimization}
\label{sec:genegraph}
% Masked flow matching addresses the \emph{optimization} side of Proposition~\ref{prop:decomp}: it restructures the training signal to make joint modeling necessary for optimality. Yet finite model capacity and noisy gradients can still produce predictions that violate known co-expression structure. We close this gap with a complementary mechanism that imposes constraints from gene-gene correlations and gene priors, guiding learning toward biologically coherent outputs. To this end, we build an explicit gene affinity graph capturing functional coupling between genes, which serves both as a scaffold for structured masking (Section~\ref{sec:structuremask}) and as a prior for output regularization (Section~\ref{sec:graphloss}).

%%% xiaohan revised
Annealed-MFM addresses the \emph{optimization} limitation identified in Proposition~\ref{prop:decomp} by restructuring the training signal to encourage joint modeling. However, finite model capacity and noisy gradients can still lead to predictions that violate known co-expression structures. 
To address this, we introduce a complementary mechanism that imposes constraints based on external functional prior and data-driven co-expression, guiding the model toward biologically coherent outputs. 
% Specifically, we construct an explicit gene affinity graph capturing functional coupling between genes, which serves both as a scaffold for structured masking (Section~\ref{sec:structuremask}) and as a prior for output regularization (Section~\ref{sec:graphloss}).

% Masked flow matching addresses the \emph{optimization} side of the limitation identified in Proposition~\ref{prop:decomp}: it restructures the training signal so that joint modeling becomes necessary for optimality. However, even with a richer training signal, finite model capacity and noisy gradients can still yield predictions that violate known co-expression structure. We therefore introduce a complementary mechanism that addresses this remaining gap by imposing constraints derived from real gene--gene correlations and gene priors, thereby guiding learning toward the correct mapping and biologically coherent configurations. This motivates building an explicit gene affinity graph that captures
% which genes are functionally coupled, serving both as a scaffold
% for structured masking (Section~\ref{sec:structuremask}) and as a prior for output regularization (Section~\ref{sec:graphloss}).

\subsubsection{Gene Affinity Graph Construction}
\label{sec:graph_construct}

We construct a gene affinity graph $\mathcal{G} = (\mathcal{V}, \mathcal{E}, S)$ over the target gene set, where $\mathcal{V}$ is the set of genes, $\mathcal{E}$ is the set of edges connecting functionally related genes, and $S \in \mathbb{R}^{G \times G}$ is the affinity matrix encoding edge weights. The graph fuses two complementary sources: STRING~\cite{szklarczyk2025string} captures functional gene-gene associations from existing databases, providing a stable prior that generalizes across datasets; we denote its affinity matrix by $S^{(\mathrm{s})}$. WGCNA-derived graph~\cite{rezaie2023pywgcna} computed from the training expression data, capturing the co-expression landscape specific to the current dataset; we denote its affinity matrix by $S^{(\mathrm{w})}$. 
% Both sources capture higher-order structure beyond pairwise correlation, linking genes that share correlated neighbors. 
We fuse them via weighted averaging:
\begin{equation}
  S = \alpha\, S^{(\mathrm{s})}
    + (1-\alpha)\, S^{(\mathrm{w})},
  \label{eq:affinity}
\end{equation}
where $S^{(\mathrm{s})}, S^{(\mathrm{w})} \in \mathbb{R}^{G \times G}$, and $\alpha \in [0,1]$ controls the balance.
The fused matrix is sparsified by retaining the top-$k$ neighbors per gene and symmetrized to yield the final edge set $\mathcal{E}$. From $S$ we derive the normalized adjacency and graph Laplacian:
\begin{equation}
  L = I - A,
   \qquad
  A = D^{-1/2} S\, D^{-1/2},
  % \qquad
  % D_{ii} = \textstyle\sum_j S_{ij}.
  \label{eq:graph_ops}
\end{equation}
where $D$ is the diagonal degree matrix with $D_{ii}=\sum_j S_{ij}$, $A$ is the symmetrically normalized adjacency matrix, and $L$ is the corresponding normalized graph Laplacian.

\subsubsection{Graph-Regularized Objective}
\label{sec:graphloss}

% Graph-structured masking reshapes the training signal to incentivize
% joint modeling (Section~\ref{sec:structuremask}), but it does not
% directly constrain the \emph{output} of the model: a network that
% learns inter-gene dependencies may still produce predictions that
% violate co-expression structure due to finite capacity or noisy
% gradients.

We use the gene affinity graph $\mathcal{G}$ to regularize the predicted expression, treating co-expression structure as a soft prior on valid outputs.  
The overall objective for graph constraints $\mathcal{L}_{\rm graph}$ combines two complementary regularizers operating at local and global scales.

\smallskip\textbf{Local graph constraint.}
Connected genes in $\mathcal{G}$ should produce similar predictions. 
% Since genes operate at different scales, we first standardize the prediction of each gene $i$ to zero mean and unit variance using running EMA statistics ($\mu_i$, $\sigma_i$), a common strategy for stabilizing normalization statistics~\cite{}:
Since genes operate at different scales, we z-score normalize each gene $i$ prediction using running EMA estimates of its mean and standard deviation ($\mu_i$, $\sigma_i$)~\cite{ioffe2015batch}: \(
z_i = (\hat{x}_{1,i} - \mu_i)/\sigma_i.\)
% \begin{equation}
%   z_i = (\hat{x}_{1,i} - \mu_i)\,/\,\sigma_i.
%   \label{eq:standardize}
% \end{equation}
For each edge $(i,j) \in \mathcal{E}$, we penalize the discrepancy $\Delta_{n,ij} = z_{ni} - z_{nj}$ at every spot $n$:
\begin{equation}
  \mathcal{L}_{\mathrm{local}} 
  = \frac{1}{N}
    \sum_{n=1}^{N}
    \sum_{(i,j) \in \mathcal{E}}
    w_{ij}\;\mathrm{Huber}_\beta(\Delta_{n,ij}),
  \qquad w_{ij} \propto S_{ij}.
  \label{eq:local}
\end{equation}
The Huber penalty \cite{huber1992robust} with threshold $\beta$ is quadratic for $|\Delta|\leq\beta$ and linear for $|\Delta|>\beta$, promoting smoothness among co-regulated genes while remaining robust to biological variation on weak edges. 
% \textcolor{red}{With the constraint of $\mathcal{L}_{\mathrm{local}}$, genes with higher correlation (i.e., larger $w_{ij}$) are more strongly coupled, enforcing their expression levels to be more similar.[Xiaohan: newly added, please check its correctness.]}
With the constraint of $\mathcal{L}_{\mathrm{local}}$, graph-connected genes with larger affinities $w_{ij}$ are more strongly coupled, encouraging more consistent standardized expression patterns.

\smallskip\textbf{Global graph constraint.}
The local term $\mathcal{L}_{\mathrm{local}}$ operates on individual edges and cannot capture graph-wide patterns. We complement it with a spectral penalty via the normalized graph  Laplacian 
(Eq.~\ref{eq:graph_ops}), where $S$ is the affinity matrix and $D$ is the diagonal degree matrix with $D_{ii} = \sum_j S_{ij}$:
\begin{equation}
  \mathcal{L}_{\mathrm{global}}
  = \frac{1}{N}
    \sum_{n=1}^{N}
    \hat{\mathbf{x}}_{1,n}^\top L\;\hat{\mathbf{x}}_{1,n}.
  \label{eq:global}
\end{equation}
Intuitively, $\mathcal{L}_{\mathrm{global}}$ penalizes predictions that oscillate rapidly between neighboring genes, encouraging the output to align with the low-frequency eigenmodes of $\mathcal{G}$, which correspond to large-scale co-expression modules.

\smallskip\textbf{Combined graph constraint.}
The two terms operate at complementary scales: $\mathcal{L}_{\mathrm{local}}$ enforces pairwise consistency, while $\mathcal{L}_{\mathrm{global}}$ suppresses graph-wide oscillations. We combine them as
$
  \mathcal{L}_{\mathrm{graph}}
  = \rho\,\mathcal{L}_{\mathrm{local}}
  + \lambda\,\mathcal{L}_{\mathrm{global}},
  \label{eq:graph_loss}
$
where $\rho$ and $\lambda$ control the local and global regularization strengths. 
Detailed hyperparameter analysis is provided in Appendix Table~\ref{tab:hyper-lossCoefficients}. The overall training objective combines the masked flow matching
loss (i.e., $\mathcal{L}_{\mathrm{MFM}}$ in Eq.~\ref{eq:masked_loss}) with the graph constraint $\mathcal{L}_{\mathrm{graph}}$.

\section{Experiments and Results}
\label{sec:exp}

% \hao{We may need to trim this section to make the information more digestible.} 

\subsection{Experimental Setup}
\label{sec:setup}

% \hao{I notice too few citations applied in the paper. now is slightly above 20. The benchmark datasets need citations?}

\textbf{Datasets.}
We evaluated our method on two public ST collections: HEST-1K~\citep{jaume2024hest} and STImage-1K4M~\citep{chen2024stimage}. 
Our main evaluation is conducted on the 10 benchmark datasets from HEST-1K, including IDC, PRAD, PAAD, SKCM, COAD, READ, CCRCC, HCC, LUNG, and LYMPH-IDC (denoted as LYMPH). 
HEST-1K~\citep{jaume2024hest} is used as the primary benchmark because it provides standardized patient-stratified splits, resulting in a leakage-controlled cross-validation setting across multiple cancer types and tissue contexts. 
In addition, we conducted supplementary experiments on the two largest human cancer categories from STImage-1K4M~\citep{chen2024stimage}, Brain and Breast, to assess scalability on larger ST collections. 
Since STImage-1K4M does not provide standardized patient-level benchmark splits comparable to HEST-1K for these categories, we followed a random train/validation/test split of 8:1:1 and report the corresponding results in Appendix Table~\ref{tab:stimage}.

\textbf{Implementation Details.}
Following prior work \cite{ICLR2025_31cc93d1}, we selected the top 50 genes (denoted as HMHVG-50) from the intersection of highly expressed and highly variable genes. UNI \cite{chen2024towards} was selected as the histology feature extractor, following prior work \citep{pmlr-v267-huang25t}.
The default optimization setup used a batch size of 2, learning rate $5\times 10^{-4}$, 100 epochs, $p_{\rm max}$ as 0.75, and $\alpha$ as 0.6, $\rho$ as 0.3, and $\lambda$ as $10^{-3}$. 
% The model uses 5 sampling steps and a 4-layer backbone with hidden dimension 128, pairwise hidden dimension 128, 4 attention heads, 8 neighbors, dropout 0.2, attention dropout 0.2, layer normalization, and SwiGLU activation. 
% Training uses MSE loss with gradient clipping at 1.0. 
All experiments are conducted on two NVIDIA RTX A5000 GPUs.
We evaluate ST prediction using Pearson correlation (PCC) and introduce Hallmark-Gene PCC (HPCC), a pathway-informed extension of gene-wise PCC. 
Given per-gene PCC $r_g$, target genes $\mathcal{G}$, and MSigDB Hallmark gene sets 
$\mathcal{P}=\{P_k\}_{k=1}^{K}$~\citep{broad2025}, we define
\begin{equation}
\mathrm{HPCC}
=
\frac{1}{|\mathcal{K}|}
\sum_{k\in\mathcal{K}}
\frac{1}{|P_k\cap\mathcal{G}|}
\sum_{g\in P_k\cap\mathcal{G}} r_g,
\quad
\mathcal{K}=\{k:\, |P_k\cap\mathcal{G}|\geq 5\}.
\label{eq:hpcc}
\end{equation}
HPCC averages gene-wise prediction accuracy over Hallmark gene sets with sufficient coverage in the evaluated target panel, providing a biologically informed complement to standard PCC.
In addition to PCC and HPCC, we report complementary results on GGC-Pearson and MSE-Mean in Appendix~\ref{sec:additional_metrics}. 
These additional metrics assess gene-gene correlation fidelity and numerical reconstruction error, respectively.

\textbf{Baselines.}
We compared our model against representative histology-to-ST baselines spanning complementary modeling paradigms. These include ST-Net~\cite{he2020integrating} as an early supervised regression-based method, BLEEP~\cite{xie2023spatially} as a contrastive learning and retrieval-based approach, TRIPLEX~\cite{chung2024accurate} for context-aware modeling with global and neighborhood histology features, Stem~\cite{ICLR2025_31cc93d1} as a diffusion-based generative model that can leverage multiple foundation-model-derived features, and STFlow~\cite{pmlr-v267-huang25t} as a flow-matching-based generative model.
Due to the high computational cost of Stem~\cite{ICLR2025_31cc93d1}, we provide a resource-bounded reproduction in the Appendix Table~\ref{tab:pcc-hmhvg50-stem}. 
Specifically, we train Stem~\cite{ICLR2025_31cc93d1} for 400 epochs and use 200 sampling steps at inference, while keeping remaining settings consistent with the original implementation. 
Other baselines are reproduced using unified training with 100 epochs for a fair comparison. Model complexity is reported in Section~\ref{sec:model-complexity}.

\begin{table*}[t]
\centering
\small
\setlength{\tabcolsep}{3pt}
\caption{
Comparison across the 10 HEST-1K benchmark datasets.
PCC and HPCC on HMHVG-50 genes are reported. 
Values are shown as $\mathrm{mean}_{\mathrm{std}}$ across folds/seeds.
}
\label{tab:pcc-hmhvg50}

\resizebox{\textwidth}{!}{
\begin{tabular}{l|cccccO|cccccO}
\toprule
\multirow{2}{*}{\textbf{Dataset}}
& \multicolumn{6}{c|}{\textbf{PCC} $\uparrow$}
& \multicolumn{6}{c}{\textbf{HPCC} $\uparrow$} \\
\cmidrule(lr){2-7} \cmidrule(lr){8-13}

& \textbf{ST-Net}
& \textbf{BLEEP}
& \textbf{BLEEP-UNI}
& \textbf{TRIPLEX}
& \textbf{STFlow}
& \textbf{Ours}

& \textbf{ST-Net}
& \textbf{BLEEP}
& \textbf{BLEEP-UNI}
& \textbf{TRIPLEX}
& \textbf{STFlow}
& \textbf{Ours} \\

\midrule
IDC   
& \mstd{0.715}{.096} 
& \mstd{0.608}{.227} 
& \mstd{0.767}{.052}
& \mstd{0.753}{.059} 
& \mstd{0.790}{.060} 
& \textbf{\mstd{0.791}{.061}}
& \mstd{0.710}{.103} 
& \mstd{0.605}{.225} 
& \mstd{0.769}{.056}
& \mstd{0.751}{.062} 
& \mstd{0.787}{.066} 
& \textbf{\mstd{0.788}{.072}} \\

PRAD  
& \mstd{0.406}{.008} 
& \mstd{0.345}{.038} 
& \mstd{0.451}{.014}
& \mstd{0.412}{.042} 
& \mstd{0.494}{.029} 
& \textbf{\mstd{0.495}{.020}}
& \mstd{0.507}{.044} 
& \mstd{0.420}{.001} 
& \mstd{0.582}{.017}
& \mstd{0.501}{.033} 
& \mstd{0.599}{.054} 
& \textbf{\mstd{0.610}{.042}} \\

PAAD  
& \mstd{0.504}{.042} 
& \mstd{0.465}{.063} 
& \mstd{0.543}{.052}
& \mstd{0.516}{.056} 
& \mstd{0.574}{.050} 
& \textbf{\mstd{0.580}{.052}}
& \mstd{0.535}{.060} 
& \mstd{0.503}{.077} 
& \mstd{0.599}{.053}
& \mstd{0.567}{.063} 
& \mstd{0.609}{.051} 
& \textbf{\mstd{0.622}{.053}} \\

SKCM  
& \mstd{0.706}{.053} 
& \mstd{0.645}{.088} 
& \mstd{0.758}{.011}
& \mstd{0.748}{.002} 
& \mstd{0.813}{.030} 
& \textbf{\mstd{0.817}{.025}}
& \mstd{0.691}{.047} 
& \mstd{0.648}{.007} 
& \mstd{0.660}{.069}
& \mstd{0.663}{.063} 
& \mstd{0.707}{.126} 
& \textbf{\mstd{0.732}{.123}} \\

COAD  
& \mstd{0.567}{.030} 
& \mstd{0.254}{.067} 
& \mstd{0.648}{.031}
& \mstd{0.579}{.035} 
& \mstd{0.649}{.055} 
& \textbf{\mstd{0.665}{.041}}
& \mstd{0.423}{.002} 
& \mstd{0.165}{.033} 
& \mstd{0.553}{.091}
& \mstd{0.535}{.027} 
& \mstd{0.574}{.060} 
& \textbf{\mstd{0.598}{.028}} \\

READ  
& \mstd{0.190}{.114} 
& \mstd{0.231}{.132} 
& \mstd{0.408}{.088}
& \mstd{0.208}{.136} 
& \mstd{0.428}{.051} 
& \textbf{\mstd{0.449}{.036}}
& \mstd{0.253}{.038} 
& \mstd{0.216}{.066} 
& \mstd{0.329}{.055}
& \mstd{0.186}{.027} 
& \mstd{0.418}{.068} 
& \textbf{\mstd{0.444}{.100}} \\

CCRCC 
& \mstd{0.317}{.071} 
& \mstd{0.306}{.058} 
& \mstd{0.391}{.069}
& \mstd{0.291}{.166} 
& \mstd{0.457}{.056} 
& \textbf{\mstd{0.470}{.059}}
& \mstd{0.332}{.075} 
& \mstd{0.312}{.071} 
& \mstd{0.385}{.080}
& \mstd{0.317}{.179} 
& \mstd{0.469}{.059} 
& \textbf{\mstd{0.476}{.065}} \\

HCC   
& \mstd{0.291}{.044} 
& \mstd{0.121}{.111} 
& \mstd{0.256}{.058}
& \mstd{0.086}{.070} 
& \mstd{0.382}{.104} 
& \textbf{\mstd{0.433}{.133}}
& \mstd{0.461}{.060} 
& \mstd{0.208}{.114} 
& \mstd{0.398}{.075}
& \mstd{0.118}{.107} 
& \mstd{0.541}{.106} 
& \textbf{\mstd{0.564}{.087}} \\

LUNG  
& \mstd{0.734}{.017} 
& \mstd{0.704}{.017} 
& \mstd{0.757}{.004}
& \mstd{0.750}{.004} 
& \mstd{0.782}{.008} 
& \textbf{\mstd{0.783}{.005}}
& \mstd{0.768}{.019} 
& \mstd{0.737}{.017} 
& \mstd{0.779}{.014}
& \mstd{0.782}{.012} 
& \mstd{0.814}{.003} 
& \textbf{\mstd{0.815}{.000}} \\

LYMPH 
& \mstd{0.545}{.130} 
& \mstd{0.511}{.109} 
& \mstd{0.564}{.118}
& \mstd{0.594}{.104} 
& \mstd{0.629}{.128} 
& \textbf{\mstd{0.644}{.123}}
& -- 
& -- 
& --
& -- 
& -- 
& -- \\

\midrule
\avgrow
\textbf{Average}
& 0.497 
& 0.419 
& 0.554
& 0.494 
& 0.600 
& \textbf{0.613}
& 0.520
& 0.424
& 0.562
& 0.491
& 0.613
& \textbf{0.628} \\

\bottomrule
\end{tabular}
}
\end{table*} % pcc: added a new column for bleep-uni
\begin{table}[t]
\centering
\small
\setlength{\tabcolsep}{9pt}
\caption{
Ablation study of our methods on the HMHVG-50. 
Average PCC across the 10 HEST-1K benchmark datasets is reported; per-dataset results are provided in Appendix Table~\ref{tab:ablation-all}.
}
\label{tab:ablation-average}
\resizebox{\textwidth}{!}{
\begin{tabular}{l|c|ccc|cc|c}
\toprule
\rowcolor{OpenSrcBlue}
\textbf{Metric}
& $w/o$ masking
& $w/o$ $\mathcal{L}_{\rm graph}$
& $w/o$ $\mathcal{L}_{\rm global}$
& $w/o$ $\mathcal{L}_{\rm local}$
& $w/o$ WGCNA
& $w/o$ STRING
& \textbf{Ours} \\
\midrule
Average PCC 
& 0.609 
& 0.605 
& 0.609
& 0.605 
& 0.607 
& 0.604 
& \textbf{0.613} \\
\bottomrule
\end{tabular}
}
\end{table} % average ablation

\subsection{Experimental Results}
\label{sec:main}
% \paragraph{Results on Hest-1k Benchmarks.} As shown in Table \ref{tab:pcc_hmhvg50}, our method achieves competitive or superior performance across datasets.
% In particular, we observe consistent improvements in Hallmark-Pathway PCC, indicating better reconstruction of biologically coherent gene programs. This may be because of the explicit and implicit modeling of gene dependencies, compared to per-gene generation.
% To better evaluate the biological validity of generated gene expression in a larger gene scope, we conducted experiments on HMHVG 100, 150, and 200 across 10 benchmarks. 
% As shown in Table \ref{tab:pcc_100_150_200}, our method consistently improves PCC and Hallmark-Gene PCC across datasets as the scope increases, suggesting its capability to capture coordinated gene expression patterns in larger scopes.

\paragraph{Quantitative Comparison with Baselines.}
Table~\ref{tab:pcc-hmhvg50} compares our method with five histology-to-ST approaches on the 10 HEST-1K benchmarks. 
Our method achieves the best average performance on both PCC and HPCC, improving over the strongest existing baseline STFlow from $0.600$ to $0.613$ in PCC and from $0.613$ to $0.628$ in HPCC. 
The gain is particularly meaningful because STFlow already substantially outperforms earlier spot-based and slide-level baselines, making this a relatively strong comparison setting. 
% Across individual datasets, the improvement is not uniform in magnitude, reflecting differences in tissue type, cohort size, and gene-expression structure, but the average trend consistently favors our method. 
The stronger gain on HPCC suggests that explicitly modeling inter-gene dependencies is beneficial not only for individual gene prediction but also for recovering biologically coordinated expression programs.
Across expanded HMHVG-100/150/200 settings as shown in the Appendix Table~\ref{tab:pcc-100-150-200}, our method consistently outperforms STFlow in average PCC and HPCC, indicating that its dependency-aware design remains effective as the target gene set broadens.
\begin{table*}[t]
\centering
\small
\setlength{\tabcolsep}{3pt}
\caption{
Hyperparameter study on masking strategy, maximum masking ratio, and alpha for graph fusion (PCC on HMHVG-50).
}
\label{tab:hyper}
\resizebox{\textwidth}{!}{
\begin{tabular}{l|ccO|cccOc|ccOc}
\toprule
\multirow{2}{*}{\textbf{Dataset}}
& \multicolumn{3}{c|}{\textbf{Masking Strategy (Learnable)}}
& \multicolumn{5}{c|}{\textbf{Maximum Masking Ratio}}
& \multicolumn{4}{c}{\textbf{Alpha for Graph Fusion}} \\
\cmidrule(lr){2-4} \cmidrule(lr){5-9} \cmidrule(lr){10-13}
& \textbf{$f(t) = 0.75$}
& \textbf{$f(t) = 1-t$}
& \textbf{$f(t) = t $}
& \textbf{$p_{max} = 0.15$}
& \textbf{$p_{max} = 0.3$}
& \textbf{$p_{max} = 0.5$}
& \textbf{$p_{max} = 0.75$}
& \textbf{$p_{max} = 0.9$}
& \textbf{$\alpha = 0.4$}
& \textbf{$\alpha = 0.5$}
& \textbf{$\alpha = 0.6$}
& \textbf{$\alpha = 0.7$} \\
\midrule
IDC   & \mstd{0.791}{.062} & \mstd{0.792}{.061} & \mstd{0.791}{.061} & \mstd{0.795}{.058} & \mstd{0.788}{.061} & \mstd{0.791}{.050} & \mstd{0.791}{.061} & \mstd{0.787}{.059} & \mstd{0.791}{.061} & \mstd{0.791}{.058} & \mstd{0.791}{.061} & \mstd{0.790}{.054} \\
PRAD  & \mstd{0.499}{.021} & \mstd{0.494}{.020} & \mstd{0.495}{.020} & \mstd{0.490}{.023} & \mstd{0.493}{.016} & \mstd{0.489}{.007} & \mstd{0.495}{.020} & \mstd{0.494}{.021} & \mstd{0.499}{.024} & \mstd{0.494}{.021} & \mstd{0.495}{.020} & \mstd{0.492}{.028} \\
PAAD  & \mstd{0.580}{.050} & \mstd{0.579}{.044} & \mstd{0.580}{.052} & \mstd{0.577}{.053} & \mstd{0.582}{.052} & \mstd{0.570}{.052} & \mstd{0.580}{.052} & \mstd{0.564}{.058} & \mstd{0.569}{.055} & \mstd{0.576}{.050} & \mstd{0.580}{.052} & \mstd{0.582}{.052} \\
SKCM  & \mstd{0.817}{.026} & \mstd{0.808}{.035} & \mstd{0.817}{.025} & \mstd{0.818}{.017} & \mstd{0.807}{.034} & \mstd{0.818}{.026} & \mstd{0.817}{.025} & \mstd{0.816}{.025} & \mstd{0.816}{.025} & \mstd{0.817}{.025} & \mstd{0.817}{.025} & \mstd{0.813}{.029} \\
COAD  & \mstd{0.666}{.041} & \mstd{0.666}{.041} & \mstd{0.665}{.041} & \mstd{0.662}{.043} & \mstd{0.653}{.055} & \mstd{0.664}{.040} & \mstd{0.665}{.041} & \mstd{0.665}{.040} & \mstd{0.665}{.040} & \mstd{0.665}{.041} & \mstd{0.665}{.041} & \mstd{0.670}{.046} \\
READ  & \mstd{0.438}{.035} & \mstd{0.416}{.044} & \mstd{0.449}{.036} & \mstd{0.421}{.048} & \mstd{0.416}{.048} & \mstd{0.422}{.049} & \mstd{0.449}{.036} & \mstd{0.417}{.057} & \mstd{0.438}{.046} & \mstd{0.430}{.038} & \mstd{0.449}{.036} & \mstd{0.446}{.026} \\
CCRCC & \mstd{0.465}{.061} & \mstd{0.462}{.053} & \mstd{0.470}{.059} & \mstd{0.456}{.052} & \mstd{0.462}{.053} & \mstd{0.446}{.067} & \mstd{0.470}{.059} & \mstd{0.468}{.056} & \mstd{0.445}{.073} & \mstd{0.457}{.056} & \mstd{0.470}{.059} & \mstd{0.456}{.064} \\
HCC   & \mstd{0.431}{.135} & \mstd{0.431}{.134} & \mstd{0.433}{.133} & \mstd{0.430}{.130} & \mstd{0.432}{.132} & \mstd{0.433}{.134} & \mstd{0.433}{.133} & \mstd{0.432}{.132} & \mstd{0.432}{.132} & \mstd{0.432}{.133} & \mstd{0.433}{.133} & \mstd{0.432}{.131} \\
LUNG  & \mstd{0.781}{.007} & \mstd{0.781}{.009} & \mstd{0.783}{.005} & \mstd{0.785}{.005} & \mstd{0.782}{.007} & \mstd{0.782}{.007} & \mstd{0.783}{.005} & \mstd{0.784}{.004} & \mstd{0.785}{.004} & \mstd{0.784}{.004} & \mstd{0.783}{.005} & \mstd{0.782}{.007} \\
LYMPH & \mstd{0.630}{.120} & \mstd{0.631}{.128} & \mstd{0.644}{.123} & \mstd{0.637}{.127} & \mstd{0.633}{.133} & \mstd{0.638}{.122} & \mstd{0.644}{.123} & \mstd{0.629}{.127} & \mstd{0.628}{.133} & \mstd{0.634}{.128} & \mstd{0.644}{.123} & \mstd{0.633}{.131} \\
\midrule
\textbf{Average} & 0.610 & 0.606 & \textbf{0.613} & 0.607 & 0.605 & 0.605 & \textbf{0.613} & 0.606 & 0.607 & 0.608 & \textbf{0.613} & 0.609 \\
\bottomrule
\end{tabular}
}
\end{table*} % hyper

Furthermore, statistical tests across all pairwise comparisons with competing methods yield p-values consistently below 0.05, indicating that our method significantly outperforms all baselines. Details are reported in Appendix Section~\ref{sec:gene-level-pcc-pvalue} and Section~\ref{sec:gene-level-hpcc-pvalue}.

% \textcolor{red}{Furthermore, paired Wilcoxon signed-rank tests on dataset-level PCC across the 10 HEST-1K datasets yielded Holm-adjusted p-values below 0.05 for all pairwise comparisons between \ourmodel~and the competing baselines, confirming that \ourmodel~significantly outperforms all baselines overall in terms of PCC. Detailed results of the overall statistical significance analysis are provided in Appendix Section~\ref{sec:statistical-overall}.}

% \input{figures/pathway_vis}
% \subsection{Gene-Level vs Pathway-Level Trade-off}
% \label{sec:tradeoff}
% We analyze the relationship between gene-level accuracy and pathway-level coherence.
% While improvements in Pearson correlation are moderate, gains in Hallmark-Pathway PCC are more pronounced.
% This suggests that our method primarily enhances biological consistency rather than merely improving point-wise prediction accuracy.
% This behavior aligns with the design of our graph regularization, which enforces smoothness and structured relationships over the gene affinity graph.

\paragraph{Ablation Study.}
\label{sec:ablation}

Table~\ref{tab:ablation-average} summarizes the ablation study of our method, reporting average PCC for HMHVG-50 over the 10 HEST-1K benchmarks.
The full model achieves the best average PCC of 0.613, suggesting that the proposed components are complementary. 
Removing annealed masking reduces the average PCC to 0.609, indicating that masking provides a useful training signal for preventing direct shortcut reconstruction and encouraging dependency-aware prediction. 
Removing the graph-based regularization term yields a larger decrease to 0.605, supporting the role of graph constraints in stabilizing gene-expression prediction under structured inter-gene dependencies. 
% Separating the two graph losses further shows that the Laplacian term mainly contributes global graph smoothness, whereas the Huber term improves robustness to local prediction errors; their combination gives the strongest overall result.
We further examined the construction of the gene graph by removing either WGCNA-derived co-expression graph or STRING-derived functional interaction priors. 
Both variants underperform the full model, with average PCCs of 0.607 and 0.604, respectively, suggesting that dataset-specific co-expression structure and external biological priors provide complementary information for graph construction. 
% Overall, these results support the combined design of conditional masking, graph-guided regularization, and hybrid biological priors.
The per-dataset results are provided in Appendix Table~\ref{tab:ablation-all}.

\paragraph{Hyperparameter Study.}
\label{sec:hyperparameter}

We conducted the hyperparameter studies with respect to the masking schedule, masking intensity, and graph-prior fusion weight, as summarized in Table~\ref{tab:hyper}. 
Among the timestep-dependent masking strategies, the proportional schedule $f(t)=t$ gives the strongest average performance, whereas the inverse schedule $f(t)=1-t$ performs worst. 
This ordering supports our motivation that stronger masking is most effective near the expression endpoint, where the model is more prone to shortcut reconstruction due to stronger target-expression signals in the interpolated input.
% This ordering is consistent with our motivation: masking is most useful near the expression endpoint, where the interpolated input contains stronger target-expression information and the model is more prone to shortcut reconstruction. 
% This empirically supports applying stronger masking near the expression endpoint, where the model is otherwise more prone to direct reconstruction and therefore benefits more from dependency-driven supervision. 
Varying the maximum masking ratio reveals a similar balance: $p_{\max}=0.75$ performs best, indicating that effective masking should be strong enough to expose inter-gene dependencies but not so aggressive that it removes the contextual signals required for stable prediction. 
For the hybrid gene graph, the best result is obtained with $\alpha=0.6$, suggesting a modest preference toward the STRING-derived functional prior while preserving the dataset-specific co-expression structure captured by WGCNA-derived graph. 
% These observations support our final configuration, which uses $f(t)=t$, $p_{\max}=0.75$, and $\alpha=0.6$ throughout the main experiments.
The hyperparameter study for the coefficients $\rho$ and $\lambda$ in the combined graph constraint introduced in Section~\ref{sec:graphloss} is reported in Appendix Table~\ref{tab:hyper-lossCoefficients}.

\begin{figure}[t]
    \centering
    \includegraphics[scale=.40]{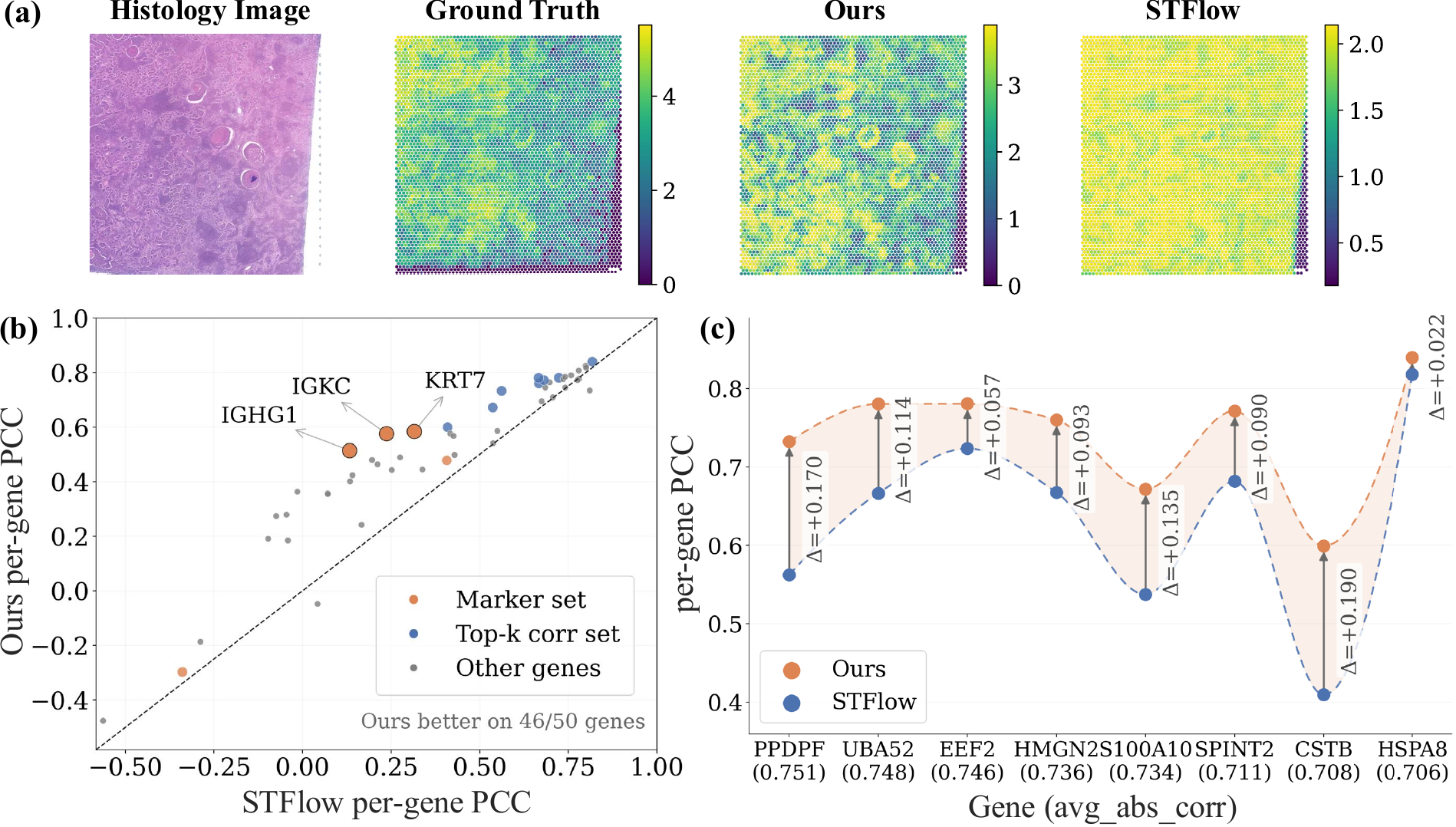}
    \caption{Visualization on LYMPH\_IDC. (a) Visualization of a biologically relevant gene KRT7 on sample NCBI68, with histology image, ground-truth ST, and predictions of \ourmodel~and STFlow (individual min-max color scales to reflect expression intensity). (b) Gene PCC comparison between \ourmodel~and STFlow (each point denotes a gene and points above the diagonal indicate improved prediction by \ourmodel). (c) PCC improvements on the top-8 genes by average absolute correlation.
    }
    \label{fig:vis}
\end{figure}

% \subsection{Qualitative Visualization}
\paragraph{Qualitative Visualization.}
\label{sec:hyperparameter}

To qualitatively assess \ourmodel, we visualize predictions from three perspectives. In all cases, our method is compared against STFlow under the same test splits and gene panels. 
% The visualization aims to evaluate not only per-gene prediction fidelity, but also the preservation of biologically meaningful spatial structure and gene-level relative performance.
We first visualize the \textbf{prediction of biologically relevant genes}. As shown in Fig.~\ref{fig:vis} (a), for the selected marker gene KRT7 in the LYMPH dataset, we show the histology image, the ground-truth ST expression, and the predictions from our method and STFlow. This comparison highlights whether the predicted ST recovers the major spatial gradients, localized high-expression regions, and tissue-specific heterogeneity observed in the histology image and ground truth. As shown in Fig.~\ref{fig:vis} (a), our method generally produces spatial patterns that are visually closer to the ground truth, with clear spatial patterns. More results on other datasets are presented in the Appendix Fig.~\ref{fig:vis-st}.
% Fig.~\ref{fig:vis} (b) is the \textbf{Gene-level PCC scatter comparison}. We summarize gene-wise prediction quality using a scatter plot in which each point corresponds to one gene, the $x$-axis denotes STFlow PCC, and the $y$-axis denotes the PCC of our method. Points above the diagonal indicate genes for which our method outperforms STFlow. This visualization gives a global overview of gene-level gains and losses and reveals whether improvements are broad-based or concentrated in a subset of genes. As shown in Fig.~\ref{fig:vis} (b), a large proportion of genes lie above the diagonal, consistent with the quantitative improvements observed in the benchmark metrics.
Fig.~\ref{fig:vis}(b) presents a \textbf{gene-level PCC scatter comparison}. Each point corresponds to one gene, where the x-axis denotes the PCC achieved by STFlow and the y-axis denotes the PCC achieved by our method. Points above the diagonal indicate genes for which our method outperforms STFlow. This visualization provides a global view of gene-level performance differences and reveals whether improvements are broadly distributed or concentrated in a subset of genes. As shown in Fig.~\ref{fig:vis}(b), a large proportion of genes lie above the diagonal, consistent with the overall quantitative gains reported in the benchmark results.
Fig.~\ref{fig:vis}(c) further visualizes the \textbf{top-$k$ genes with high gene-gene correlation}. To focus on genes that are strongly coupled with the overall transcriptomic program, we ranked genes using the ground-truth gene-gene correlation matrix from the test set. Specifically, for each gene $g$, we computed the average absolute correlation with all other genes:
\(
\mathrm{AvgAbsCorr}(g) = \frac{1}{G-1} \sum_{h \neq g} \left| \mathrm{Corr}(g,h) \right|.
\)
We then selected the top-$k$ genes with the highest $\mathrm{AvgAbsCorr}$ values and compared the PCC achieved by our method and STFlow on these genes. As shown in Fig.~\ref{fig:vis}(c), our method consistently achieves higher PCC on many highly correlated genes, suggesting improved preservation of dependency-relevant expression structure.
% Taken together, these visualizations show that the proposed method improves ST prediction in three respects: it better reconstructs spatial expression patterns of biologically relevant genes, yields broader gene-level gains over STFlow, and more effectively predicts genes with strong transcriptomic connectivity. 
These qualitative trends are consistent with the observed improvements in PCC, further supporting the effectiveness of our method.

% \textbf{Biologically relevant gene visualization.}
% \emph{Spatial visualization of representative biologically relevant genes for each cancer type. For each gene, we show the ground-truth spatial transcriptomic expression together with the predictions from our method and the original STFlow baseline. The selected genes are curated from prior cancer-specific literature, including reported marker genes, biomarker genes, and other disease-associated genes, and are restricted to genes covered by the evaluated cancer-specific panel. PCC values are reported for the two prediction methods.}

% \textbf{Gene-level PCC scatter plot.}
% \emph{Gene-wise comparison of prediction quality between our method and STFlow for each cancer type. Each point denotes one gene, with the $x$-axis showing the PCC of STFlow and the $y$-axis showing the PCC of our method. Points above the diagonal correspond to genes on which our method outperforms STFlow. Marker genes, top-$k$ correlation genes, and their intersection can be highlighted for interpretation.}

% \textbf{Top-$k$ high-correlation gene comparison.}
% \emph{Comparison of per-gene prediction accuracy for the top-$k$ genes ranked by ground-truth average absolute gene--gene correlation. Each row summarizes a highly connected gene selected using only the ground-truth correlation structure, and compares the PCC of our method against STFlow. This visualization highlights whether improved modeling of gene dependency structure translates into better prediction of transcriptomically central genes.}

\section{Conclusion}
\label{sec:conclusion}
We present \ourmodel, a correlation-guided conditional flow matching framework for histology-to-ST generation. 
Rather than treating genes as independent prediction targets, \ourmodel~explicitly promotes cross-gene joint modeling through \textit{annealed masked flow matching} and \textit{gene graph-regularized optimization}. 
The annealed masking strategy encourages the model to infer masked gene expression from related genes, while the gene affinity graph integrates external functional priors and data-driven co-expression structure to impose local consistency and global smoothness on generated ST profiles. 
Across 10 patient-stratified HEST-1K benchmarks and two supplementary STImage-1K4M datasets, \ourmodel~achieves the best average PCC and HPCC among evaluated methods. 
These results suggest that modeling intrinsic gene-gene dependencies is a promising direction for more biologically coherent histology-to-ST prediction.
The limitations and future work are presented in the Appendix Section~\ref{sec:limitations}.

{\small
\bibliographystyle{splncs04}
\bibliography{neuips_2026}
}

%%%%%%%%%%%%%%%%%%%%%%%%%%%%%%%%%%%%%%%%%%%%%%%%%%%%%%%%%%%%

\newpage
\appendix

\vspace{2em}

\noindent {\Large{Appendix}}

\section{Methodology Extension}

\subsection{Proof of Proposition 2.}
\label{sec:proof_proposition2}

 \begin{proof}
  The masked input $\tilde{x}_t$ is shared across all output dimensions, and $\tilde{x}_t$ restricted to $\mathcal{M}$ contains only fixed mask tokens carrying no sample-specific information, so the optimal predictor for every gene $g$ is $v^{*,g} = \mathbb{E}[u_t^g \mid x_t^{(\mathcal{O})}, t, c]$. For $g \in \mathcal{M}$, expanding via the law of total expectation:
  \begin{equation}
    \mathbb{E}\bigl[u_t^g \mid x_t^{(\mathcal{O})}, t, c\bigr]
    = \int u_t^g \;
      p\bigl(x_1^{(\mathcal{M})} \mid x_t^{(\mathcal{O})}, t, c\bigr)
      \;\mathrm{d}x_1^{(\mathcal{M})}.
  \end{equation}
  This integral can only be estimated if the model captures the co-variation between $x_1^{(\mathcal{M})}$ and $x_1^{(\mathcal{O})}$; without such knowledge the model has no basis for predicting $u_t^g$ beyond the unconditional prior. 
\end{proof}

\subsection{Convergence analysis.}
\label{sec:convergence}
A natural concern is whether training under the masked objective
$\mathcal{L}_{\mathrm{MFM}}$ still yields a valid generative model,
given that the network is optimized on corrupted inputs $\tilde{x}_t$
but deployed on clean inputs $x_t$ at inference.
We show that minimizing $\mathcal{L}_{\mathrm{MFM}}$ guarantees
convergence of $\mathcal{L}_{\mathrm{CFM}}$, which governs
inference-time generation. 

\begin{theorem}[Convergence of Annealed-MFM]
\label{thm:convergence}
  Let $q(\mathcal{M})$ be any masking distribution over subsets of
  $\{1,\dots,G\}$. Then:
  \begin{enumerate}
\item[{(i)}]
      \textbf{Lower bound.}\;
      $\mathcal{L}_{\mathrm{MFM}}(\theta;\mathcal{M})
      \ge \mathcal{L}_{\mathrm{CFM}}(\theta)$
      for all $\theta$ and all $\mathcal{M}$.
\item[{(ii)}]
      \textbf{Shared minimizer.}\;
      There exists $\theta^*$ that minimizes
      $\mathcal{L}_{\mathrm{MFM}}$ and simultaneously drives
      $\mathcal{L}_{\mathrm{CFM}}$ to its global minimum.
      Consequently,
      $\;\mathcal{L}_{\mathrm{MFM}}(\theta;\mathcal{M}) \to 0$ lead to $\mathcal{L}_{\mathrm{CFM}}(\theta) \to 0$,
      and the learned flow converges to the target distribution
      under standard regularity conditions~\citep{lipman2023flow}.
  \end{enumerate}
\end{theorem}

\begin{proof}
  \textit{(i)}~Both losses supervise all $G$ genes against the same
  target $u_t$ and differ only in the network input
  ($\tilde{x}_t$ vs.\ $x_t$).
  Since $\tilde{x}_t$ is a deterministic degradation of $x_t$,
  the data-processing inequality gives, for each gene $g$,
  \begin{equation}
    \mathbb{E}\bigl[
      (v_\theta^g(\tilde{x}_t,t,c) - u_t^g)^2
    \bigr]
    \;\ge\;
    \mathbb{E}\bigl[
      (v_\theta^g(x_t,t,c) - u_t^g)^2
    \bigr].
  \end{equation}
  Summing over $g$ yields the bound.

  \textit{(ii)}~Let $\theta^*$ satisfy
  $v_{\theta^*}(x,t,c) = \mathbb{E}[u_t \mid x,t,c]$ for all
  inputs $x$ over a sufficiently expressive function.
  Then $v_{\theta^*}(\tilde{x}_t,t,c)
  = \mathbb{E}[u_t \mid \tilde{x}_t,t,c]$,
  which is the pointwise minimizer of $\mathcal{L}_{\mathrm{MFM}}$.
  Since $\theta^*$ simultaneously minimizes both objectives,
  and $\mathcal{L}_{\mathrm{CFM}}(\theta^*) = 0$ under standard
  flow matching guarantees~\citep{lipman2023flow}, the global
  minimum of $\mathcal{L}_{\mathrm{MFM}}$ also drives
  $\mathcal{L}_{\mathrm{CFM}}$ to zero.
\end{proof}

\subsection{Exploratory Extension: Graph-Structured Masking}
\label{sec:structuremask}

Proposition~\ref{prop:joint} establishes that masking induces joint
modeling, but the quality of the learning signal depends on
\emph{which} genes are masked together. 
Under independent Bernoulli masking, masked genes are scattered
across unrelated pathways, and the joint conditional
$p(x_1^{(\mathcal{M})} \mid x_1^{(\mathcal{O})}, c)$ factorizes
approximately, weakening the inter-gene signal that masking is designed to exploit.

Instead,  we leverage the gene affinity graph $\mathcal{G}$ to construct structured mask sets that respect biological modularity. At each training step, we sample a set of seed genes uniformly at random and expand each seed to a connected subgraph via $r$-hop neighborhood sampling on $\mathcal{G}$.  This ensures co-regulated genes are masked \emph{together}, forcing the model to reconstruct entire expression modules from regulatory context and maximizing the joint-modeling benefit of Proposition~\ref{prop:joint}.

\textbf{Parameters.} The number of seeds is set to $\lceil p(t) \cdot G / \bar{n}_r \rceil$, where $\bar{n}_r$ is the average $r$-hop neighborhood size in $\mathcal{G}$, so that the expected number of masked genes matches $p(t) \cdot G$. We fix $r = 1$ throughout all experiments; the masking ratio $p(t)$ then controls the number of seeds rather than the neighborhood depth. If sampled neighborhoods overlap, the union is taken as the final mask set $\mathcal{M}$.

\textbf{Results.} Results and analysis are provided in the Appendix Table~\ref{tab:graph_masking_variant} and Appendix Section~\ref{sec:results_GraphStructuredMasking}.

\subsection{Training Procedure}
\label{sec:train}

The overall training objective combines the masked flow matching
loss (Eq.~\ref{eq:masked_loss}) with the graph regularizer
(Eq.~\ref{eq:graph_loss}):
\begin{equation}
  \mathcal{L}
  = \mathcal{L}_{\mathrm{MFM}}
  + \mathcal{L}_{\mathrm{graph}}.
  \label{eq:overall_loss}
\end{equation}
$\mathcal{L}_{\mathrm{MFM}}$ supervises all $G$ genes under
partial observation, incentivizing joint modeling
(Proposition~\ref{prop:joint});
$\mathcal{L}_{\mathrm{graph}}$ constrains the output toward
biologically coherent expression profiles
(Section~\ref{sec:graphloss}). The full procedure is summarized in Algorithm~\ref{alg:training}.

\begin{algorithm}[t]
\caption{\ourmodel~Training}
\label{alg:training}
\begin{algorithmic}[1]
\Require Gene affinity graph $\mathcal{G}$,  
         max masking ratio $p_{\max}$, learning rate $\eta$, $x_0\sim p_0$, $x_1\sim p_1$
\For{each training iteration}
  \State $x_t \leftarrow (1-t)\,x_0 + t\,x_1$, \;
         $t \sim \mathrm{Uniform}(0,1)$  \Comment{Interpolation} %  , \;         $x_0 \sim p_0$
         % \vspace{1mm}
\State $\mathcal{M} \sim q(\mathcal{M} \mid \mathcal{G},\, p(t))$, \; $p(t) \leftarrow p_{\max} \cdot t$ 
    \Comment{Graph mask  (\S\ref{sec:structuremask}), Eq.~\ref{eq:mask_rate}}
  
  \State $\tilde{x}_{t,g} \leftarrow
    \begin{cases}
      m_g & g \in \mathcal{M} \\
      x_{t,g} & g \notin \mathcal{M}
    \end{cases}$
    \Comment{Masking, Eq.~\ref{eq:mask_op}}
    
  \State $v_t \leftarrow v_\theta(\tilde{x}_t, t, c)$
    \Comment{Prediction}
     \State 
         $\mathcal{L} \leftarrow
         \|v_t - (x_1{-}x_0)\|^2
         + \rho\,\mathcal{L}_{\mathrm{local}}
         + \lambda\,\mathcal{L}_{\mathrm{global}}$   \Comment{Eq.~\ref{eq:masked_loss}, Eq.~\ref{eq:graph_loss}}
  % \State $\mathcal{L}_{\mathrm{MFM}} \leftarrow
  %   \frac{1}{NG}\sum_{n,g}(\hat{x}_{1,ng} - x_{1,ng})^2$
   
  % \State $\mathcal{L}_{\mathrm{graph}} \leftarrow
  %   \rho\,\mathcal{L}_{\mathrm{local}}
  %   + \lambda\,\mathcal{L}_{\mathrm{global}}$
  %   \Comment{}
  \State $\theta \leftarrow \theta
    - \eta\,\nabla_\theta
    (\mathcal{L})$   \Comment{Update}
\EndFor
\end{algorithmic}
\end{algorithm}

\begin{table*}[t]
\centering
\small
\renewcommand{\arraystretch}{1.05}
\setlength{\tabcolsep}{4pt}
\caption{Comparison between STFlow and ours on HEST-1K benchmarks under HMHVG-100/150/200 settings, with PCC and HPCC reported.}
\label{tab:pcc-100-150-200}

\resizebox{\textwidth}{!}{
\begin{tabular}{lcOcOcO|cOcOcO}
\toprule
\multirow{2}{*}{\textbf{Dataset}}
& \multicolumn{2}{c}{\textbf{PCC-100} $\uparrow$}
& \multicolumn{2}{c}{\textbf{PCC-150} $\uparrow$}
& \multicolumn{2}{c|}{\textbf{PCC-200} $\uparrow$}
& \multicolumn{2}{c}{\textbf{HPCC-100} $\uparrow$}
& \multicolumn{2}{c}{\textbf{HPCC-150} $\uparrow$}
& \multicolumn{2}{c}{\textbf{HPCC-200} $\uparrow$} \\
\cmidrule(lr){2-3}
\cmidrule(lr){4-5}
\cmidrule(lr){6-7}
\cmidrule(lr){8-9}
\cmidrule(lr){10-11}
\cmidrule(lr){12-13}
& \textbf{STFlow} & \textbf{Ours}
& \textbf{STFlow} & \textbf{Ours}
& \textbf{STFlow} & \textbf{Ours}
& \textbf{STFlow} & \textbf{Ours}
& \textbf{STFlow} & \textbf{Ours}
& \textbf{STFlow} & \textbf{Ours} \\
\midrule

IDC
& \mstd{0.772}{.050} & \mstd{0.778}{.052}
& \mstd{0.735}{.048} & \mstd{0.741}{.053}
& \mstd{0.707}{.052} & \mstd{0.705}{.055}
& \mstd{0.777}{.040} & \mstd{0.786}{.042}
& \mstd{0.756}{.039} & \mstd{0.760}{.047}
& \mstd{0.722}{.048} & \mstd{0.719}{.052} \\

PRAD
& \mstd{0.480}{.026} & \mstd{0.475}{.005}
& \mstd{0.475}{.028} & \mstd{0.479}{.008}
& \mstd{0.473}{.024} & \mstd{0.473}{.007}
& \mstd{0.497}{.030} & \mstd{0.484}{.004}
& \mstd{0.484}{.025} & \mstd{0.487}{.001}
& \mstd{0.470}{.019} & \mstd{0.466}{.001} \\

PAAD
& -- & --
& -- & --
& -- & --
& -- & --
& -- & --
& -- & -- \\

SKCM
& \mstd{0.793}{.020} & \mstd{0.805}{.014}
& \mstd{0.801}{.003} & \mstd{0.810}{.002}
& \mstd{0.788}{.009} & \mstd{0.789}{.001}
& \mstd{0.809}{.035} & \mstd{0.820}{.031}
& \mstd{0.812}{.004} & \mstd{0.822}{.005}
& \mstd{0.797}{.010} & \mstd{0.799}{.004} \\

COAD
& \mstd{0.602}{.061} & \mstd{0.604}{.053}
& \mstd{0.580}{.056} & \mstd{0.581}{.064}
& \mstd{0.547}{.063} & \mstd{0.557}{.049}
& \mstd{0.615}{.063} & \mstd{0.607}{.051}
& \mstd{0.570}{.075} & \mstd{0.570}{.077}
& \mstd{0.553}{.081} & \mstd{0.560}{.065} \\

READ
& \mstd{0.368}{.043} & \mstd{0.372}{.043}
& \mstd{0.336}{.032} & \mstd{0.356}{.021}
& \mstd{0.337}{.038} & \mstd{0.349}{.053}
& \mstd{0.364}{.013} & \mstd{0.371}{.031}
& \mstd{0.328}{.003} & \mstd{0.351}{.025}
& \mstd{0.342}{.004} & \mstd{0.348}{.027} \\

CCRCC
& \mstd{0.450}{.049} & \mstd{0.463}{.052}
& \mstd{0.436}{.062} & \mstd{0.433}{.058}
& \mstd{0.433}{.054} & \mstd{0.442}{.054}
& \mstd{0.456}{.037} & \mstd{0.462}{.051}
& \mstd{0.460}{.058} & \mstd{0.451}{.057}
& \mstd{0.447}{.056} & \mstd{0.456}{.055} \\

HCC
& \mstd{0.392}{.091} & \mstd{0.425}{.113}
& \mstd{0.329}{.047} & \mstd{0.353}{.066}
& \mstd{0.356}{.063} & \mstd{0.378}{.082}
& \mstd{0.393}{.091} & \mstd{0.402}{.098}
& \mstd{0.328}{.043} & \mstd{0.336}{.055}
& \mstd{0.379}{.064} & \mstd{0.384}{.079} \\

LUNG
& \mstd{0.731}{.010} & \mstd{0.730}{.012}
& \mstd{0.685}{.013} & \mstd{0.690}{.014}
& -- & --
& \mstd{0.747}{.022} & \mstd{0.747}{.024}
& \mstd{0.716}{.019} & \mstd{0.721}{.021}
& -- & -- \\

LYMPH
& \mstd{0.668}{.118} & \mstd{0.670}{.118}
& \mstd{0.639}{.101} & \mstd{0.647}{.098}
& \mstd{0.661}{.102} & \mstd{0.655}{.107}
& \mstd{0.690}{.109} & \mstd{0.690}{.111}
& \mstd{0.665}{.095} & \mstd{0.675}{.089}
& \mstd{0.674}{.095} & \mstd{0.670}{.099} \\

\midrule
\textbf{Average}
& 0.584 & \textbf{0.591}
& 0.557 & \textbf{0.566}
& 0.538 & \textbf{0.543}
& 0.594 & \textbf{0.597}
& 0.569 & \textbf{0.575}
& 0.548 & \textbf{0.550} \\
\bottomrule
\end{tabular}
}
\end{table*}
\begin{table*}[t]
\centering
\small
\setlength{\tabcolsep}{4pt}
\caption{Ablation study of different graph priors and module design (PCC on HMHVG-50).}
\label{tab:ablation-all}
\resizebox{\textwidth}{!}{
\begin{tabular}{l|cccccccccc|O}
\toprule
\textbf{Ablation} 
& \textbf{IDC} 
& \textbf{PRAD} 
& \textbf{PAAD} 
& \textbf{SKCM} 
& \textbf{COAD} 
& \textbf{READ} 
& \textbf{CCRCC} 
& \textbf{HCC} 
& \textbf{LUNG} 
& \textbf{LYMPH} 
& \textbf{Avg.} \\
\midrule
$w/o$ masking 
& \mstd{0.790}{.063} 
& \mstd{0.490}{.029} 
& \mstd{0.579}{.049} 
& \mstd{0.813}{.021} 
& \textbf{\mstd{0.666}{.048}} 
& \mstd{0.434}{.057} 
& \mstd{0.461}{.055} 
& \mstd{0.428}{.122} 
& \mstd{0.780}{.008} 
& \textbf{\mstd{0.645}{.123}} 
& 0.609 \\
$w/o$ $\mathcal{L}_{\rm graph}$ 
& \textbf{\mstd{0.794}{.059}} 
& \mstd{0.496}{.033} 
& \mstd{0.577}{.049} 
& \textbf{\mstd{0.818}{.029}} 
& \mstd{0.647}{.050} 
& \mstd{0.451}{.038}
& \mstd{0.460}{.049} 
& \mstd{0.398}{.098} 
& \mstd{0.781}{.008} 
& \mstd{0.632}{.125} 
& 0.605 \\
$w/o$ $\mathcal{L}_{\rm Local}$ 
& \mstd{0.789}{.054}
& \mstd{0.497}{.021}
& \mstd{0.574}{.056}
& \mstd{0.811}{.026}
& \mstd{0.649}{.048}
& \mstd{0.451}{.038}
& \mstd{0.462}{.048}
& \mstd{0.408}{.107}
& \mstd{0.780}{.009}
& \mstd{0.633}{.128}
& 0.605 \\
$w/o$ $\mathcal{L}_{\rm Global}$ 
& \mstd{0.792}{.060}
& \mstd{0.497}{.022}
& \mstd{0.573}{.058}
& \mstd{0.815}{.026}
& \mstd{0.663}{.042}
& \textbf{\mstd{0.452}{.019}}
& \mstd{0.459}{.058}
& \mstd{0.430}{.131}
& \mstd{0.782}{.007}
& \mstd{0.631}{.133}
& 0.609 \\
\midrule
$w/o$ WGCNA
& \mstd{0.792}{.056} 
& \mstd{0.491}{.023} 
& \mstd{0.576}{.051} 
& \mstd{0.805}{.034} 
& \mstd{0.650}{.053} 
& \mstd{0.448}{.030} 
& \mstd{0.459}{.052} 
& \textbf{\mstd{0.436}{.128}} 
& \textbf{\mstd{0.783}{.006}} 
& \mstd{0.630}{.123} 
& 0.607 \\
$w/o$ STRING
& \mstd{0.788}{.063} 
& \textbf{\mstd{0.502}{.020}} 
& \mstd{0.575}{.050} 
& \mstd{0.809}{.037} 
& \mstd{0.651}{.054} 
& \mstd{0.425}{.045} 
& \mstd{0.456}{.061} 
& \mstd{0.430}{.132} 
& \mstd{0.781}{.006} 
& \mstd{0.626}{.127} 
& 0.604 \\
\midrule
Ours
& \mstd{0.791}{.061} 
& \mstd{0.495}{.020} 
& \textbf{\mstd{0.580}{.052}} 
& \mstd{0.817}{.025} 
& \mstd{0.665}{.041} 
& \mstd{0.449}{.036} 
& \textbf{\mstd{0.470}{.059}} 
& \mstd{0.433}{.133} 
& \textbf{\mstd{0.783}{.005}} 
& \mstd{0.644}{.123} 
& \textbf{0.613} \\
\bottomrule
\end{tabular}
}
\end{table*}
\begin{table*}[t]
\centering
\small
\setlength{\tabcolsep}{5pt}
\caption{
Hyperparameter study of the Local and Global regularization coefficient on HMHVG-50 genes.
PCC across the 10 HEST-1K benchmark datasets is reported as $\mathrm{mean}_{\mathrm{std}}$, where the subscript denotes the standard deviation across folds/seeds.
}
\label{tab:hyper-lossCoefficients}
\resizebox{0.6\textwidth}{!}{
\begin{tabular}{l|cOc|cOc}
\toprule
\multirow{2}{*}{\textbf{Dataset}}
& \multicolumn{3}{c|}{\textbf{Local Coefficient} $\mathbf{\rho}$}
& \multicolumn{3}{c}{\textbf{Global Coefficient} $\mathbf{\lambda}$} \\
\cmidrule(lr){2-4} \cmidrule(lr){5-7}
& \textbf{0.1}
& \textbf{0.3}
& \textbf{0.5}
& $\mathbf{10^{-4}}$
& $\mathbf{10^{-3}}$
& $\mathbf{10^{-2}}$ \\
\midrule
IDC   
& \mstd{0.792}{.058}
& \mstd{0.791}{.061}
& \textbf{\mstd{0.793}{.058}}
& \textbf{\mstd{0.794}{.057}}
& \mstd{0.791}{.061}
& \mstd{0.789}{.063} \\

PRAD  
& \textbf{\mstd{0.499}{.024}}
& \mstd{0.495}{.020}
& \mstd{0.495}{.019}
& \mstd{0.494}{.020}
& \mstd{0.495}{.020}
& \textbf{\mstd{0.497}{.014}} \\

PAAD  
& \mstd{0.575}{.056}
& \textbf{\mstd{0.580}{.052}}
& \mstd{0.566}{.055}
& \mstd{0.565}{.050}
& \textbf{\mstd{0.580}{.052}}
& \mstd{0.553}{.066} \\

SKCM  
& \textbf{\mstd{0.825}{.019}}
& \mstd{0.817}{.025}
& \mstd{0.815}{.031}
& \mstd{0.814}{.026}
& \mstd{0.817}{.025}
& \textbf{\mstd{0.821}{.023}} \\

COAD  
& \mstd{0.651}{.059}
& \textbf{\mstd{0.665}{.041}}
& \textbf{\mstd{0.665}{.048}}
& \mstd{0.664}{.042}
& \textbf{\mstd{0.665}{.041}}
& \mstd{0.656}{.039} \\

READ  
& \mstd{0.444}{.024}
& \textbf{\mstd{0.449}{.036}}
& \mstd{0.447}{.042}
& \mstd{0.437}{.042}
& \textbf{\mstd{0.449}{.036}}
& \mstd{0.444}{.020} \\

CCRCC 
& \mstd{0.462}{.061}
& \textbf{\mstd{0.470}{.059}}
& \mstd{0.444}{.064}
& \mstd{0.438}{.067}
& \textbf{\mstd{0.470}{.059}}
& \mstd{0.451}{.057} \\

HCC   
& \mstd{0.419}{.124}
& \textbf{\mstd{0.433}{.133}}
& \mstd{0.428}{.127}
& \mstd{0.431}{.131}
& \mstd{0.433}{.133}
& \textbf{\mstd{0.446}{.136}} \\

LUNG  
& \mstd{0.781}{.008}
& \textbf{\mstd{0.783}{.005}}
& \mstd{0.779}{.009}
& \mstd{0.782}{.006}
& \textbf{\mstd{0.783}{.005}}
& \mstd{0.776}{.008} \\

LYMPH 
& \mstd{0.634}{.123}
& \textbf{\mstd{0.644}{.123}}
& \mstd{0.629}{.124}
& \mstd{0.640}{.125}
& \mstd{0.644}{.123}
& \textbf{\mstd{0.647}{.119}} \\

\midrule
\textbf{Average}
& 0.608
& \textbf{0.613}
& 0.606
& 0.606
& \textbf{0.613}
& 0.608 \\
\bottomrule
\end{tabular}
}
\end{table*}

\begin{figure}[t]
    \centering
    \includegraphics[width=\textwidth]{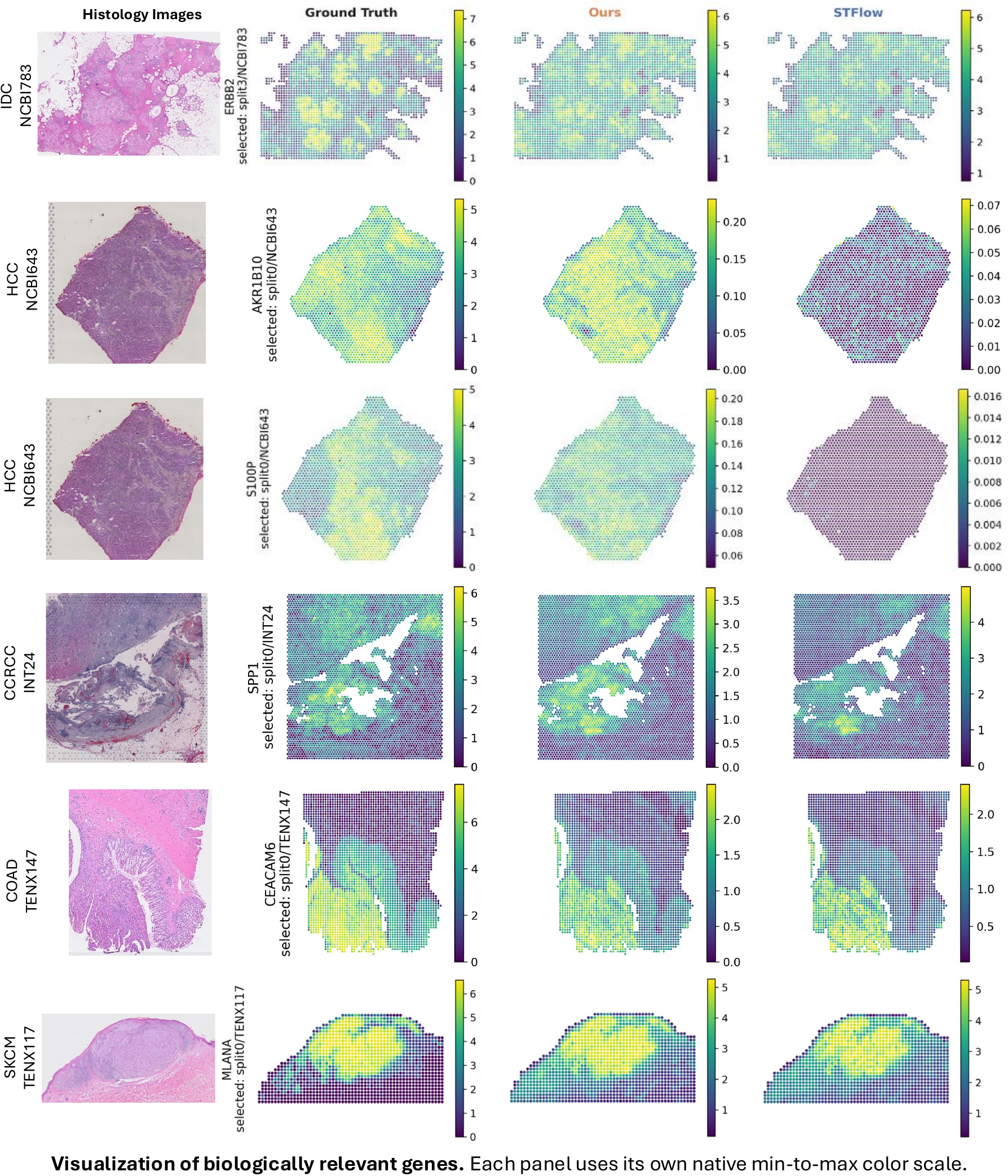}
    \caption{Visualization of Marker Genes on HEST-1K benchmark.}
    \label{fig:vis-st}
\end{figure}

\begin{table*}[t]
\centering
\small
\setlength{\tabcolsep}{4pt}
\caption{
Comparison across the 10 HEST-1K benchmark datasets.
GGC-Pearson on HMHVG-50 genes is reported.
Values are shown as $\mathrm{mean}_{\mathrm{std}}$, where the subscript denotes the standard deviation across folds/seeds.
}
\label{tab:ggc-pearson}

\resizebox{0.70\textwidth}{!}{
\begin{tabular}{l|ccccccO}
\toprule
\textbf{Dataset}
& \textbf{ST-Net}
& \textbf{BLEEP}
& \textbf{BLEEP-UNI}
& \textbf{TRIPLEX}
& \textbf{Stem}
& \textbf{STFlow}
& \textbf{Ours} \\
\midrule
IDC   
& \mstd{0.742}{.055}
& \mstd{0.691}{.088}
& \mstd{0.774}{.068}
& \mstd{0.781}{.059}
& \mstd{0.711}{.062}
& \mstd{0.793}{.064}
& \textbf{\mstd{0.796}{.054}} \\

PRAD  
& \textbf{\mstd{0.851}{.019}}
& \mstd{0.843}{.003}
& \mstd{0.815}{.085}
& \mstd{0.836}{.046}
& \mstd{0.828}{.009}
& \mstd{0.802}{.063}
& \mstd{0.778}{.075} \\

PAAD  
& \mstd{0.545}{.053}
& \mstd{0.498}{.027}
& \mstd{0.553}{.051}
& \mstd{0.581}{.125}
& \mstd{-0.025}{.012}
& \textbf{\mstd{0.685}{.097}}
& \mstd{0.624}{.102} \\

SKCM  
& \mstd{0.625}{.015}
& \mstd{0.696}{.044}
& \textbf{\mstd{0.761}{.094}}
& \mstd{0.725}{.120}
& \mstd{0.001}{.003}
& \mstd{0.749}{.110}
& \mstd{0.759}{.104} \\

COAD  
& \mstd{0.664}{.049}
& \mstd{0.539}{.001}
& \textbf{\mstd{0.694}{.080}}
& \mstd{0.666}{.028}
& \mstd{0.088}{.135}
& \mstd{0.673}{.063}
& \mstd{0.687}{.044} \\

READ  
& \mstd{0.830}{.015}
& \mstd{0.748}{.102}
& \mstd{0.645}{.213}
& \mstd{0.829}{.030}
& \mstd{0.011}{.012}
& \mstd{0.867}{.043}
& \textbf{\mstd{0.876}{.049}} \\

CCRCC 
& \mstd{0.584}{.095}
& \mstd{0.562}{.141}
& \mstd{0.643}{.101}
& \mstd{0.592}{.103}
& \mstd{0.634}{.116}
& \mstd{0.657}{.108}
& \textbf{\mstd{0.669}{.091}} \\

HCC   
& \mstd{0.271}{.054}
& \mstd{0.131}{.087}
& \textbf{\mstd{0.276}{.034}}
& \mstd{0.140}{.107}
& \mstd{0.000}{.013}
& \mstd{0.181}{.037}
& \mstd{0.240}{.124} \\

LUNG  
& \mstd{0.581}{.040}
& \mstd{0.608}{.052}
& \mstd{0.649}{.034}
& \mstd{0.578}{.007}
& \mstd{0.015}{.006}
& \mstd{0.656}{.019}
& \textbf{\mstd{0.706}{.002}} \\

LYMPH 
& \mstd{0.364}{.144}
& \mstd{0.434}{.102}
& \mstd{0.472}{.148}
& \textbf{\mstd{0.481}{.155}}
& \mstd{0.111}{.144}
& \mstd{0.261}{.178}
& \mstd{0.341}{.225} \\

\midrule
\avgrow
\textbf{Average}
& 0.606
& 0.575
& 0.628
& 0.621
& 0.237
& 0.633
& \textbf{0.648} \\

\bottomrule
\end{tabular}
}
\end{table*}
\begin{table*}[t]
\centering
\small
\setlength{\tabcolsep}{5pt}
\caption{
Comparison across the 10 HEST-1K benchmark datasets.
MSE-Mean on HMHVG-50 genes is reported.
Values are shown as $\mathrm{mean}_{\mathrm{std}}$, where the subscript denotes the standard deviation across folds/seeds.
}
\label{tab:mse}

\resizebox{0.70\textwidth}{!}{
\begin{tabular}{l|cccccO}
\toprule
\textbf{Dataset}
& \textbf{ST-Net}
& \textbf{BLEEP}
& \textbf{BLEEP-UNI}
& \textbf{TRIPLEX}
& \textbf{STFlow}
& \textbf{Ours} \\
\midrule
IDC   
& \mstd{2.313}{3.245}
& \mstd{2.763}{3.385}
& \mstd{2.260}{3.366}
& \mstd{2.108}{3.275}
& \mstd{2.288}{3.182}
& \textbf{\mstd{2.013}{2.799}} \\

PRAD  
& \mstd{1.551}{0.729}
& \mstd{2.020}{0.922}
& \mstd{1.696}{0.736}
& \mstd{1.595}{0.817}
& \mstd{1.452}{0.865}
& \textbf{\mstd{1.409}{0.713}} \\

PAAD  
& \mstd{1.332}{0.875}
& \mstd{1.401}{0.897}
& \mstd{1.222}{0.770}
& \mstd{1.279}{0.847}
& \textbf{\mstd{1.123}{0.724}}
& \mstd{1.141}{0.774} \\

SKCM  
& \mstd{2.894}{2.187}
& \mstd{3.311}{1.838}
& \mstd{2.565}{1.805}
& \mstd{2.853}{1.805}
& \mstd{2.206}{1.765}
& \textbf{\mstd{2.182}{1.620}} \\

COAD  
& \mstd{3.436}{2.347}
& \mstd{5.370}{3.833}
& \textbf{\mstd{3.006}{2.405}}
& \mstd{3.545}{2.824}
& \mstd{3.184}{2.680}
& \mstd{3.061}{2.341} \\

READ  
& \mstd{2.679}{2.010}
& \mstd{2.457}{2.053}
& \mstd{2.010}{1.651}
& \mstd{2.673}{1.960}
& \textbf{\mstd{1.564}{1.095}}
& \mstd{1.611}{1.143} \\

CCRCC 
& \mstd{1.655}{1.098}
& \mstd{1.938}{1.050}
& \mstd{1.709}{0.987}
& \mstd{1.582}{1.043}
& \mstd{1.467}{0.803}
& \textbf{\mstd{1.321}{0.764}} \\

HCC   
& \mstd{4.892}{3.684}
& \mstd{5.384}{3.935}
& \mstd{4.682}{3.855}
& \mstd{4.767}{3.627}
& \mstd{4.366}{3.529}
& \textbf{\mstd{4.163}{3.346}} \\

LUNG  
& \mstd{1.691}{1.780}
& \mstd{1.799}{1.947}
& \mstd{1.676}{2.036}
& \mstd{1.727}{1.938}
& \mstd{1.539}{2.013}
& \textbf{\mstd{1.530}{2.229}} \\

LYMPH 
& \mstd{2.444}{3.007}
& \mstd{3.174}{3.766}
& \mstd{2.487}{3.250}
& \mstd{2.238}{3.133}
& \mstd{2.000}{2.726}
& \textbf{\mstd{1.902}{2.844}} \\

\midrule
\avgrow
\textbf{Average}
& 2.489
& 2.962
& 2.331
& 2.436
& 2.119
& \textbf{2.033} \\

\bottomrule
\end{tabular}
}
\end{table*}
\begin{table}[t]
\centering
\small
\setlength{\tabcolsep}{8pt}
\caption{
Supplementary evaluation on the two largest human cancer categories from STImage-1K4M.
PCC and HPCC are reported as $\mathrm{mean}_{\mathrm{std}}$.
}
\label{tab:stimage}
\resizebox{0.5\textwidth}{!}{
\begin{tabular}{l|cO|cO}
\toprule
\multirow{2}{*}{\textbf{Dataset}}
& \multicolumn{2}{c|}{\textbf{Pearson} $\uparrow$}
& \multicolumn{2}{c}{\textbf{HPCC} $\uparrow$} \\
\cmidrule(lr){2-3} \cmidrule(lr){4-5}
& \textbf{STFlow}
& \textbf{Ours}
& \textbf{STFlow}
& \textbf{Ours} \\
\midrule
Brain
& \mstd{0.546}{.264}
& \textbf{\mstd{0.558}{.245}}
& --
& -- \\
Breast
& \mstd{0.513}{.173}
& \textbf{\mstd{0.525}{.160}}
& \mstd{0.464}{.093}
& \textbf{\mstd{0.473}{.081}} \\
\midrule
Avg.
& 0.530
& \textbf{0.542}
& 0.464
& \textbf{0.473} \\
\bottomrule
\end{tabular}
}
\end{table}
\begin{table}[t]
\centering
\small
\setlength{\tabcolsep}{7pt}
\caption{
Pearson correlation comparison between STFlow, a Graph-Structured Masking variant of our method, and our main model.
Results are reported on the 10 HEST-1K benchmarks and the two supplementary STImage-1K4M cancer categories.
Values are shown as $\mathrm{mean}_{\mathrm{std}}$, where the subscript denotes the standard deviation across folds/seeds.
}
\label{tab:graph_masking_variant}
\resizebox{0.5\textwidth}{!}{
\begin{tabular}{l|ccO}
\toprule
\textbf{Dataset}
& \textbf{STFlow}
& \textbf{Graph-Structured Masking}
& \textbf{Ours} \\
\midrule
\multicolumn{4}{c}{\textit{HEST-1K benchmarks}} \\
\midrule
IDC   
& \mstd{0.790}{.060}
& \textbf{\mstd{0.797}{.056}}
& \mstd{0.791}{.061} \\

PRAD  
& \mstd{0.494}{.029}
& \textbf{\mstd{0.497}{.019}}
& \mstd{0.495}{.020} \\

PAAD  
& \mstd{0.574}{.050}
& \mstd{0.577}{.049}
& \textbf{\mstd{0.580}{.052}} \\

SKCM  
& \mstd{0.813}{.030}
& \mstd{0.805}{.030}
& \textbf{\mstd{0.817}{.025}} \\

COAD  
& \mstd{0.649}{.055}
& \mstd{0.664}{.041}
& \textbf{\mstd{0.665}{.041}} \\

READ  
& \mstd{0.428}{.051}
& \mstd{0.438}{.064}
& \textbf{\mstd{0.449}{.036}} \\

CCRCC 
& \mstd{0.457}{.056}
& \mstd{0.449}{.062}
& \textbf{\mstd{0.470}{.059}} \\

HCC   
& \mstd{0.382}{.104}
& \mstd{0.431}{.129}
& \textbf{\mstd{0.433}{.133}} \\

LUNG  
& \mstd{0.782}{.008}
& \mstd{0.782}{.008}
& \textbf{\mstd{0.783}{.005}} \\

LYMPH 
& \mstd{0.629}{.128}
& \mstd{0.642}{.121}
& \textbf{\mstd{0.644}{.123}} \\

\midrule
\textbf{Average}
& 0.600
& 0.608
& \textbf{0.613} \\

\midrule
\multicolumn{4}{c}{\textit{STImage-1K4M supplementary categories}} \\
\midrule
Brain
& \mstd{0.546}{.264}
& \mstd{0.556}{.238}
& \textbf{\mstd{0.558}{.245}} \\

Breast
& \mstd{0.513}{.173}
& \mstd{0.511}{.182}
& \textbf{\mstd{0.525}{.160}} \\

\midrule
\textbf{Average}
& 0.530
& 0.534
& \textbf{0.542} \\

\bottomrule
\end{tabular}
}
\end{table}
\begin{table*}[t]
\centering
\small
\setlength{\tabcolsep}{2.5pt}
\caption{
Comparison across the 10 HEST-1K benchmark datasets.
PCC and HPCC on the HMHVG-50 are reported. 
Values are reported as $\mathrm{mean}_{\mathrm{std}}$, where the subscript denotes the standard deviation across folds/seeds.
% \hao{And I would suggest times everything by 10x, eg $0.715_{0.096}$ ->  $7.15_{0.96} $, what make it more comparable for the demonstration and the metrics is marked as eg PCC (1.0E1). }
}
\label{tab:pcc-hmhvg50-stem}

\resizebox{\textwidth}{!}{
\begin{tabular}{l|ccccccO|ccccccO}
\toprule
\multirow{2}{*}{\textbf{Dataset}}
& \multicolumn{7}{c|}{\textbf{PCC} $\uparrow$}
& \multicolumn{7}{c}{\textbf{HPCC} $\uparrow$} \\
\cmidrule(lr){2-8} \cmidrule(lr){9-15}

& \textbf{ST-Net}
& \textbf{BLEEP}
& \textbf{BLEEP-UNI}
& \textbf{TRIPLEX}
& \textbf{Stem}
& \textbf{STFlow}
& \textbf{Ours}

& \textbf{ST-Net}
& \textbf{BLEEP}
& \textbf{BLEEP-UNI}
& \textbf{TRIPLEX}
& \textbf{Stem}
& \textbf{STFlow}
& \textbf{Ours} \\

\midrule
IDC   
& \mstd{0.715}{.096} 
& \mstd{0.608}{.227} 
& \mstd{0.767}{.052}
& \mstd{0.753}{.059} 
& \mstd{0.612}{.056} 
& \mstd{0.790}{.060} 
& \textbf{\mstd{0.791}{.061}}
& \mstd{0.710}{.103} 
& \mstd{0.605}{.225} 
& \mstd{0.769}{.056}
& \mstd{0.751}{.062} 
& \mstd{0.587}{.056} 
& \mstd{0.787}{.066} 
& \textbf{\mstd{0.788}{.072}} \\

PRAD  
& \mstd{0.406}{.008} 
& \mstd{0.345}{.038} 
& \mstd{0.451}{.014}
& \mstd{0.412}{.042} 
& \mstd{0.227}{.038} 
& \mstd{0.494}{.029} 
& \textbf{\mstd{0.495}{.020}}
& \mstd{0.507}{.044} 
& \mstd{0.420}{.001} 
& \mstd{0.582}{.017}
& \mstd{0.501}{.033} 
& \mstd{0.256}{.090} 
& \mstd{0.599}{.054} 
& \textbf{\mstd{0.610}{.042}} \\

PAAD  
& \mstd{0.504}{.042} 
& \mstd{0.465}{.063} 
& \mstd{0.543}{.052}
& \mstd{0.516}{.056} 
& \mstd{0.006}{.004} 
& \mstd{0.574}{.050} 
& \textbf{\mstd{0.580}{.052}}
& \mstd{0.535}{.060} 
& \mstd{0.503}{.077} 
& \mstd{0.599}{.053}
& \mstd{0.567}{.063} 
& \mstd{0.009}{.006} 
& \mstd{0.609}{.051} 
& \textbf{\mstd{0.622}{.053}} \\

SKCM  
& \mstd{0.706}{.053} 
& \mstd{0.645}{.088} 
& \mstd{0.758}{.011}
& \mstd{0.748}{.002} 
& \mstd{0.003}{.001} 
& \mstd{0.813}{.030} 
& \textbf{\mstd{0.817}{.025}}
& \mstd{0.691}{.047} 
& \mstd{0.648}{.007} 
& \mstd{0.660}{.069}
& \mstd{0.663}{.063} 
& \mstd{-0.008}{.017} 
& \mstd{0.707}{.126} 
& \textbf{\mstd{0.732}{.123}} \\

COAD  
& \mstd{0.567}{.030} 
& \mstd{0.254}{.067} 
& \mstd{0.648}{.031}
& \mstd{0.579}{.035} 
& \mstd{0.048}{.042} 
& \mstd{0.649}{.055} 
& \textbf{\mstd{0.665}{.041}}
& \mstd{0.423}{.002} 
& \mstd{0.165}{.033} 
& \mstd{0.553}{.091}
& \mstd{0.535}{.027} 
& \mstd{0.054}{.047} 
& \mstd{0.574}{.060} 
& \textbf{\mstd{0.598}{.028}} \\

READ  
& \mstd{0.190}{.114} 
& \mstd{0.231}{.132} 
& \mstd{0.408}{.088}
& \mstd{0.208}{.136} 
& \mstd{-0.001}{.004} 
& \mstd{0.428}{.051} 
& \textbf{\mstd{0.449}{.036}}
& \mstd{0.253}{.038} 
& \mstd{0.216}{.066} 
& \mstd{0.329}{.055}
& \mstd{0.186}{.027} 
& \mstd{0.001}{.002} 
& \mstd{0.418}{.068} 
& \textbf{\mstd{0.444}{.100}} \\

CCRCC 
& \mstd{0.317}{.071} 
& \mstd{0.306}{.058} 
& \mstd{0.391}{.069}
& \mstd{0.291}{.166} 
& \mstd{0.257}{.089} 
& \mstd{0.457}{.056} 
& \textbf{\mstd{0.470}{.059}}
& \mstd{0.332}{.075} 
& \mstd{0.312}{.071} 
& \mstd{0.385}{.080}
& \mstd{0.317}{.179} 
& \mstd{0.257}{.106} 
& \mstd{0.469}{.059} 
& \textbf{\mstd{0.476}{.065}} \\

HCC   
& \mstd{0.291}{.044} 
& \mstd{0.121}{.111} 
& \mstd{0.256}{.058}
& \mstd{0.086}{.070} 
& \mstd{-0.002}{.003} 
& \mstd{0.382}{.104} 
& \textbf{\mstd{0.433}{.133}}
& \mstd{0.461}{.060} 
& \mstd{0.208}{.114} 
& \mstd{0.398}{.075}
& \mstd{0.118}{.107} 
& \mstd{0.004}{.006} 
& \mstd{0.541}{.106} 
& \textbf{\mstd{0.564}{.087}} \\

LUNG  
& \mstd{0.734}{.017} 
& \mstd{0.704}{.017} 
& \mstd{0.757}{.004}
& \mstd{0.750}{.004} 
& \mstd{0.004}{.000} 
& \mstd{0.782}{.008} 
& \textbf{\mstd{0.783}{.005}}
& \mstd{0.768}{.019} 
& \mstd{0.737}{.017} 
& \mstd{0.779}{.014}
& \mstd{0.782}{.012} 
& \mstd{0.006}{.002} 
& \mstd{0.814}{.003} 
& \textbf{\mstd{0.815}{.000}} \\

LYMPH 
& \mstd{0.545}{.130} 
& \mstd{0.511}{.109} 
& \mstd{0.564}{.118}
& \mstd{0.594}{.104} 
& \mstd{0.253}{.080} 
& \mstd{0.629}{.128} 
& \textbf{\mstd{0.644}{.123}}
& -- 
& -- 
& --
& -- 
& --
& -- 
& -- \\

\midrule
\avgrow
\textbf{Average}
& 0.497 
& 0.419 
& 0.554
& 0.494 
& 0.141 
& 0.600 
& \textbf{0.613}
& 0.520
& 0.424
& 0.562
& 0.491
& 0.129
& 0.613
& \textbf{0.628} \\

\bottomrule
\end{tabular}
}
\end{table*}

% p-value
\begin{table}[t]
\centering
\caption{Gene-level PCC significance analysis on HEST-1K. For each dataset, gene-wise PCC values were first averaged across cross-validation folds, yielding 50 gene-level PCC values per dataset. These values were then pooled across all 10
datasets, producing 500 dataset-gene pairs for each pairwise comparison. We compared \ourmodel~against each baseline using paired Wilcoxon signed-rank tests, followed by Holm correction. Mean Diff. denotes \ourmodel~minus the baseline.}
\label{tab:pooled-gene-pcc-significance}
\resizebox{0.6\textwidth}{!}{
\begin{tabular}{lccccc}
\toprule
Baseline & Mean Diff. & Wins & Losses & Raw $p$ & Holm $p$ \\
\midrule
ST-Net    & 0.1152 & 479 & 21  & $4.16\times10^{-79}$ & $1.66\times10^{-78}$ \\
BLEEP     & 0.1935 & 497 & 3   & $2.33\times10^{-83}$ & $1.17\times10^{-82}$ \\
BLEEP-UNI & 0.0584 & 424 & 76  & $9.15\times10^{-55}$ & $1.83\times10^{-54}$ \\
TRIPLEX   & 0.1189 & 475 & 25  & $1.99\times10^{-78}$ & $5.98\times10^{-78}$ \\
STFlow    & 0.0129 & 330 & 170 & $1.49\times10^{-16}$ & $1.49\times10^{-16}$ \\
\bottomrule
\end{tabular}
}
\end{table} %gene-level pcc p-value
\begin{table}[t]
  \centering
  \caption{Pooled Hallmark-Gene PCC significance analysis on HEST-1K. For each dataset, pathway-level Hallmark-Gene PCC values were first aggregated across cross-validation folds by matching pathway names and averaging the corresponding gene-level
  pathway PCC values when a pathway was present in multiple folds. The resulting dataset-pathway pairs were then pooled across all datasets and compared between \ourmodel~and each baseline using paired Wilcoxon signed-rank tests, followed by Holm
  correction. Mean Diff. denotes \ourmodel~minus the baseline.}
  \label{tab:pooled-hallmark-gene-significance}
  \resizebox{0.6\textwidth}{!}{
  \begin{tabular}{lccccc}
  \toprule
  Baseline & Mean Diff. & Wins & Losses & Raw $p$ & Holm $p$ \\
  \midrule
  ST-Net    & 0.1153 & 28 & 0 & $7.45\times10^{-9}$  & $3.73\times10^{-8}$ \\
  BLEEP     & 0.1921 & 28 & 0 & $7.45\times10^{-9}$  & $3.73\times10^{-8}$ \\
  BLEEP-UNI & 0.0598 & 27 & 1 & $1.49\times10^{-8}$  & $3.73\times10^{-8}$ \\
  TRIPLEX   & 0.1145 & 28 & 0 & $7.45\times10^{-9}$  & $3.73\times10^{-8}$ \\
  STFlow    & 0.0112 & 22 & 6 & $9.32\times10^{-5}$  & $9.32\times10^{-5}$ \\
  \bottomrule
  \end{tabular}
  }
  \end{table} %gene-level hpcc p-value

% model complexity
\begin{table}[t]
  \centering
  \caption{Model complexity comparison on the HEST benchmark. \#Params counts the trainable parameters used under the current benchmark training setup, excluding external frozen or precomputed feature extractors. Avg. Time / Epoch is averaged over all
  dataset-fold runs of each method.}
  \label{tab:model-complexity}
  \resizebox{0.65\textwidth}{!}{
  \begin{tabular}{lccc}
  \toprule
  Method & \#Params (M) & Model Size (MB) & Avg. Time / Epoch (min) \\
  \midrule
  ST-Net    & 11.202 & 42.73  & 0.301 \\
  BLEEP     & 24.178 & 92.23  & 0.922 \\
  BLEEP-UNI & 38.205 & 145.74 & 8.342 \\
  TRIPLEX   & 53.908 & 205.64 & 1.433 \\
  Stem      & 35.581 & 135.73 & 0.230 \\
  STFlow    & 1.148  & 4.38   & 0.038 \\
  \ourmodel (Our)  & 1.148  & 4.38   & 0.038 \\
  \bottomrule
  \end{tabular}
  }
  \end{table}

\section{Evaluation Metrics}

\label{sec:metrics}

We evaluate histology-to-ST prediction from two complementary perspectives: expression accuracy and dependency-structure fidelity. Expression accuracy is measured using Pearson correlation and Hallmark-Gene PCC, while dependency-structure fidelity is measured using a gene-gene correlation preservation metric (GGC-Pearson). As an auxiliary expression-level error metric, we also report MSE-Mean. 
% Additional metrics are provided in Appendix~\ref{app:metrics}.

\paragraph{Pearson correlation coefficient (PCC).}
We first evaluate gene-level prediction fidelity using the PCC, computed independently for each gene across spatial spots. Given predicted expression $\hat{y}_{:,g}$ and ground-truth expression $y_{:,g}$ for gene $g$, the PCC is defined as
\begin{equation}
r_g = \frac{\mathrm{cov}(\hat{y}_{:,g}, y_{:,g})}{\sigma(\hat{y}_{:,g}) \, \sigma(y_{:,g})}.
\end{equation}
We report the mean PCC over all evaluated genes:
\begin{equation}
\mathrm{Pearson} = \frac{1}{|\mathcal{G}|} \sum_{g \in \mathcal{G}} r_g,
\end{equation}
where $\mathcal{G}$ denotes the set of target genes. This metric captures how well the model preserves spatial variation patterns for individual genes.

\paragraph{Hallmark-Gene PCC (HPCC).}
To assess biological consistency on functionally meaningful genes, we adopt a pathway-informed evaluation based on Hallmark gene sets from MSigDB \cite{broad2025}. 
The equation is shown in Equation \ref{eq:hpcc}.
% Let $\mathcal{P}=\{P_m\}_{m=1}^M$ denote the collection of Hallmark gene sets, and let
% \begin{equation}
% \tilde{P}_m = P_m \cap \mathcal{G}
% \end{equation}
% be the subset covered by the evaluated gene panel. Following the implementation, we evaluate only Hallmark sets with at least five covered genes, namely
% \begin{equation}
% \mathcal{M}=\{m \mid |\tilde{P}_m| \geq 5\}.
% \end{equation}
% For each valid Hallmark set, we compute the average gene-wise PCC over the covered genes:
% \begin{equation}
% s_m^{\mathrm{hallmark}} = \frac{1}{|\tilde{P}_m|} \sum_{g \in \tilde{P}_m} r_g,
% \end{equation}
% and report the mean across all valid Hallmark sets:
% \begin{equation}
% \mathrm{Hallmark\mbox{-}Gene} = \frac{1}{|\mathcal{M}|} \sum_{m \in \mathcal{M}} s_m^{\mathrm{hallmark}}.
% \end{equation}
Compared with Pearson, this metric emphasizes prediction quality on genes associated with curated biological programs.

\paragraph{Gene-Gene Correlation (GGC).}
Our method is motivated by the observation that gene expression is not independent across genes, but instead exhibits structured gene-gene dependencies. To evaluate whether a model preserves this dependency structure, we compare the gene-gene correlation matrices induced by the predicted and ground-truth transcriptomes across spatial spots.

Let $\hat{Y}, Y \in \mathbb{R}^{N \times G}$ denote the predicted and ground-truth expression matrices over $N$ spots and $G$ target genes. We compute Pearson gene-gene correlation matrices across spatial spots using the full evaluated gene set:
\begin{equation}
\hat{C}_{g,h} = \mathrm{corr}(\hat{y}_{:,g}, \hat{y}_{:,h}), \qquad
C_{g,h} = \mathrm{corr}(y_{:,g}, y_{:,h}).
\end{equation}
We then extract the upper-triangular entries (excluding the diagonal), vectorize them, and compute the Pearson correlation between the predicted and ground-truth vectors:
\begin{equation}
\mathrm{GGC\mbox{-}Pearson} = \mathrm{corr}\!\left(\mathrm{vec}_{\triangle}(\hat{C}),\; \mathrm{vec}_{\triangle}(C)\right).
\end{equation}
For numerical stability, undefined correlation values arising from constant-expression genes are set to zero in the implementation. This metric measures the global extent to which the predicted transcriptome preserves real gene-gene dependency structure.

% GGC-MAE is excluded!!
% As a supplementary absolute-error metric on the same dependency structure, we also compute
% \begin{equation}
% \mathrm{GGC\mbox{-}MAE} =
% \frac{1}{|\mathcal{E}|}
% \sum_{(g,h)\in \mathcal{E}}
% \left| \hat{C}_{g,h} - C_{g,h} \right|,
% \end{equation}
% where $\mathcal{E}$ denotes the set of upper-triangular gene pairs. Unlike GGC-Pearson, which measures global structural agreement, GGC-MAE directly quantifies the average absolute mismatch in pairwise gene-gene correlation strength.

% GeneCorrTopK-F1 is excluded as well!!
% As a supporting metric, we also evaluate recovery of the strongest gene-gene dependencies using GeneCorrTopK-F1. Specifically, we rank gene pairs by the absolute values of their correlation strengths in $\hat{C}$ and $C$, select the top
% \begin{equation}
% K = \max\!\left(1,\left\lceil 0.05 \cdot |\mathcal{E}| \right\rceil \right),
% \end{equation}
% where $|\mathcal{E}|$ is the number of valid upper-triangular gene pairs, and compute precision, recall, and F1-score between the predicted and ground-truth top-$K$ edge sets. This metric complements GGC-Pearson by focusing on recovery of the strongest gene-gene relationships rather than the full correlation structure.

\paragraph{Mean Squared Error (MSE).}
In addition to correlation-based metrics, we report the MSE to measure the absolute deviation between predicted and ground-truth gene expression values in Table~\ref{tab:mse}. 
For a given gene $g$, let $\hat{y}_{ig}$ and $y_{ig}$ denote the predicted and ground-truth expression values at spot $i$, respectively, and let $N$ be the number of test spots. The gene-wise MSE is defined as
\begin{equation}
\mathrm{MSE}_g = \frac{1}{N}\sum_{i=1}^{N}\left(\hat{y}_{ig} - y_{ig}\right)^2.
\end{equation}
We then summarize performance across genes by averaging the gene-wise MSE values:
\begin{equation}
\mathrm{MSE\text{-}Mean} = \frac{1}{G}\sum_{g=1}^{G}\mathrm{MSE}_g,
\end{equation}
where $G$ is the number of evaluated genes. Lower MSE-Mean indicates better predictive accuracy, as it reflects smaller absolute discrepancies between predicted and observed expression profiles.

% MSE-Q2 is excluded as well!!
% \paragraph{MSE-Q2.}
% As an auxiliary error-based metric, we further report MSE-Q2, defined as the median of gene-wise mean squared errors across the target genes. For each gene $g$, the mean squared error is
% \begin{equation}
% e_g = \frac{1}{N} \sum_{i=1}^{N} \left(\hat{y}_{i,g} - y_{i,g}\right)^2,
% \end{equation}
% where $N$ is the number of spatial spots. We then summarize the typical error across genes using the median:
% \begin{equation}
% \mathrm{MSE\mbox{-}Q2} = \mathrm{median}_{g \in \mathcal{G}} \; e_g.
% \end{equation}
% We use the median rather than the mean because gene-wise errors are often skewed, and the median provides a more robust estimate of typical prediction error.

Overall, higher values indicate better performance for PCC, HPCC, and GGC-Pearson, whereas lower values are better for MSE-Mean.

\section{Supplementary Results}
\label{sec:supplementary_results}

\subsection{Evaluation on Expanded Gene Sets on the HEST-1K Benchmark.}
To further evaluate whether the advantage persists beyond a small target-gene set, Table~\ref{tab:pcc-100-150-200} reports results on expanded HMHVG-100, HMHVG-150, and HMHVG-200 settings. 
Compared with STFlow, our method maintains a higher average PCC across all three gene scopes, with average PCC improving from $0.584$ to $0.591$ for HMHVG-100, from $0.557$ to $0.566$ for HMHVG-150, and from $0.538$ to $0.543$ for HMHVG-200. 
A similar pattern is observed for HPCC, where our method improves the average score from $0.594$ to $0.597$, $0.569$ to $0.575$, and $0.548$ to $0.550$ under the HMHVG-100, HMHVG-150, and HMHVG-200 settings, respectively. 
Although the absolute margins are moderate, the improvements are stable across increasingly broad gene sets, supporting the effectiveness of dependency-aware generation when the prediction target expands from highly variable marker genes to larger biologically relevant gene scopes. 
For PAAD, these expanded-gene evaluations are omitted because fewer than 100 selected genes are available after preprocessing.

\subsection{Results with Additional Metrics on HEST-1K Benchmark}
\label{sec:additional_metrics}

In the main text, we report PCC and HPCC as the primary evaluation metrics. Here, we further provide GGC-Pearson and MSE-Mean to assess complementary aspects of prediction quality. GGC-Pearson measures the agreement between the predicted and ground-truth gene-gene correlation structures, thereby evaluating whether a method preserves transcriptomic dependency patterns beyond per-gene prediction. MSE-Mean measures the mean squared prediction error averaged over genes, reflecting the numerical fidelity of the generated expression values.

As shown in Table~\ref{tab:ggc-pearson}, our method achieves the highest average GGC-Pearson of $0.648$, outperforming STFlow ($0.633$) and BLEEP-UNI ($0.628$). Although the best-performing method varies across individual datasets, the strongest average result indicates that our dependency-aware design better preserves global gene-gene correlation structure across the HEST-1K benchmark. 
Table~\ref{tab:mse} further shows that our method obtains the lowest average MSE-Mean of $2.033$, compared with $2.119$ for STFlow and $2.331$ for BLEEP-UNI, and achieves the best MSE-Mean on most datasets. These results suggest that the \ourmodel~improves not only correlation-based prediction quality, but also the numerical reconstruction fidelity of spatial gene expression.

\subsection{Results on STImage-1K4M}
Table~\ref{tab:stimage} reports supplementary evaluation on the two largest human cancer categories from STImage-1K4M~\cite{chen2024stimage}. 
Our method improves PCC over STFlow on both Brain and Breast, increasing the average PCC from $0.530$ to $0.542$. 
For HPCC, only Breast satisfies the Hallmark gene-set coverage criterion, where our method improves the score from $0.464$ to $0.473$. 
These results are consistent with the HEST-1K findings and suggest that the proposed dependency-aware design remains beneficial on larger ST collections.

\subsection{Results for Exploratory Extension with Graph-Structured Masking}
\label{sec:results_GraphStructuredMasking}
Table~\ref{tab:graph_masking_variant} reports the empirical results of this variant. 
Graph-structured masking improves over STFlow on average, increasing PCC from $0.600$ to $0.608$ on the 10 HEST-1K benchmarks and from $0.530$ to $0.534$ on the two STImage-1K4M categories. 
However, it remains below our main model, which achieves $0.613$ and $0.542$, respectively. 
We hypothesize that masking entire graph neighborhoods can make each masked block relatively large, which weakens the fine-grained timestep-dependent control provided by annealed random masking. 
As a result, the structured variant provides useful but limited gains, further supporting the importance of the proposed annealed masking design in the main model.

\subsection{Gene-Level Statistical Significance on HEST-1K}
\label{sec:gene-level-pcc-pvalue}
To provide a finer-grained complementary analysis beyond dataset-level significance testing, we performed a pooled gene-level PCC analysis on HEST-1K. Specifically, for each dataset, we first averaged each gene's PCC across cross-validation folds, resulting in 50 gene-level PCC values per dataset. Pooling these values across all 10 datasets yielded 500 dataset-gene pairs, which were then compared between \ourmodel~and each baseline using paired Wilcoxon signed-rank tests with Holm correction. As shown in Table~\ref{tab:pooled-gene-pcc-significance}, \ourmodel~significantly outperforms all baselines under this pooled gene-level analysis, with adjusted $p$-values far below 0.05 in every comparison. The strongest improvements are observed against BLEEP and TRIPLEX, while the comparison with STFlow remains significant but with a smaller effect size, indicating that STFlow is the most competitive baseline under this analysis.

\subsection{Hallmark-Gene Statistical Significance on HEST-1K}
\label{sec:gene-level-hpcc-pvalue}
We further performed an HPCC significance analysis on HEST-1K. Specifically, for each dataset, we matched Hallmark pathways across cross-validation folds by pathway name and averaged the corresponding pathway-level gene-PCC values when the same pathway appeared in multiple folds. These dataset-pathway pairs were then pooled across datasets and compared between \ourmodel~and each baseline using paired Wilcoxon signed-rank tests with Holm correction. As shown in Table~\ref{tab:pooled-hallmark-gene-significance}, \ourmodel~significantly outperforms all baselines under this pooled Hallmark-Gene analysis, with particularly large gains over BLEEP and TRIPLEX. The comparison with STFlow remains significant, although the effect size is substantially smaller than for the other baselines.

\subsection{Model Complexity}
\label{sec:model-complexity}
Table~\ref{tab:model-complexity} summarizes model complexity under the current HEST benchmark training setup. We report the number of trainable parameters, the corresponding fp32 model size, and the average training time per epoch. Notably, \ourmodel~and STFlow have the same trainable parameter count because both methods use the same trainable denoising backbone, while differing in loss design and regularization. Compared with the larger end-to-end or partial-finetuning baselines, \ourmodel~maintains a smaller trainable footprint and a very low per-epoch training cost.

\section{Limitations and Future Work}
\label{sec:limitations}

Despite the strong empirical performance of \ourmodel, several limitations remain. 
First, although HEST-1K provides a diverse cross-platform benchmark with standardized patient-stratified splits, our evaluation does not include a fully independent external cohort. 
Such external validation remains essential for assessing clinical robustness under variations in tissue preparation, sequencing platform, staining protocol, scanner type, and cohort composition. 
Future work will therefore incorporate independently collected ST cohorts to more rigorously evaluate out-of-distribution generalization.

Second, although \ourmodel~achieves the best overall performance across the evaluated datasets, the current prediction quality remains insufficient for direct clinical translation. 
Histology-to-ST generation is intrinsically challenging, as molecular states cannot always be reliably inferred from morphology alone and may be affected by cohort size, tissue heterogeneity, technical noise, and gene-panel coverage. 
Improving robustness in clinically relevant settings will require stronger dependency-aware modeling and higher-quality ST training data.

Finally, our current evaluation focuses primarily on reconstruction fidelity and transcriptomic structure preservation. 
A key next step is to assess whether the generated ST profiles provide measurable benefits for downstream biological and clinical analyses, such as tissue-domain identification, pathway activity estimation, biomarker discovery, patient stratification, and prognosis modeling. 
Demonstrating utility in such downstream tasks is important for establishing the practical value of histology-conditioned ST generation beyond benchmark-level prediction metrics.

%%%%%%%%%%%%%%%%%%%%%%%%%%%%%%%%%%%%%%%%%%%%%%%%%%%%%%%%%%%%

% \newpage
% \input{checklist.tex}

\end{document}